\RequirePackage{fix-cm}
 \documentclass{article}

\usepackage[main, final]{neurips_2026}

 \usepackage[utf8]{inputenc} 
\usepackage[T1]{fontenc}    
\usepackage{url}            
\usepackage{booktabs}       
\usepackage{amsfonts}       
\usepackage{nicefrac}       
\usepackage{microtype}      
\usepackage{xcolor}         
\usepackage{titletoc}

\usepackage{tikz}
\usetikzlibrary{calc,patterns,shapes.geometric}
\usepackage[textsize=tiny]{todonotes}
\usepackage{global_aliases}

\title{On the Complexity of Preference-Based Bandits}

\author{%
  Ahmed Ben Yahmed\thanks{Correspondence to: a.benyahmed@criteo.com.} 
  \\
  Criteo AI Lab, Paris, France\\
  CREST, ENSAE, France\\
  FairPlay joint team\\
  \And
  Marc Abeille \\
  Criteo AI Lab, Paris, France\\
  FairPlay joint team\\
  \And
  Clément Calauzènes \\
 Criteo AI Lab, Paris, France\\
  FairPlay joint team\\
}

\begin{document}

\maketitle

\begin{abstract}


We study preference-based bandits with general reward function classes, where a learner 
sequentially selects pairs of arms and observes binary preference feedback governed by the 
Bradley--Terry model. This setting naturally arises in applications such as recommender 
systems, tournament ranking, and learning from human feedback, where relative preferences 
are easier to elicit than absolute rewards. The observation model inherits the logistic 
bandit challenge of handling the problem-dependent constant $\kappa$, which accounts for 
the non-linearity of the link function and can grow arbitrarily large. Moreover, prior work 
has predominantly focused on linear or kernelized reward models, precluding the use of 
richer function classes. To address these limitations, we consider general reward function 
classes and introduce the \emph{locally sensitive eluder dimension}, a novel 
complexity measure tailored to the logistic structure of preference feedback that yields 
fine-grained regret guarantees without unfavorable dependence on $\kappa$. Building on 
this notion, we propose \textbf{GINOP} (Generic INformative OPtimism), an algorithm that 
constructs log-loss confidence sets and jointly selects arm pairs to balance optimism and 
informative exploration. We establish a first-order regret bound that, in contrast with what previous results 
suggest, demonstrates that learning with preference feedback is as statistically efficient 
as learning from direct reward observation.  Finally, we corroborate our theoretical findings with empirical 
evaluations against competitive baselines.

\end{abstract}
\section{Introduction}
In contrast to classical stochastic bandits, where the learner plays arms and receives direct reward 
feedback~\citep{lattimore2020bandit}, dueling bandits -- also referred to as 
preference-based bandits -- constitute a sequential decision-making framework in which 
the learner plays pairs of arms (duels) 
and observes noisy preference feedback~\citep{Dueling_Survey}. Specifically, at each 
round, the learner receives a binary signal indicating which arm in the selected pair is 
preferred. Preference-based bandits naturally arise in real-world applications where 
eliciting absolute rewards is difficult, while pairwise comparisons are more natural and 
reliable. Notable examples include online recommendation systems -- where users tend to 
provide more accurate relative preferences than absolute ratings -- web search ranking, 
fine-tuning of large language models, and comparative evaluations such as choosing between 
two restaurants or movies. More broadly, this framework is well-suited to any setting in 
which relative judgments are easier to obtain than absolute reward values.

When the decision space -- comprising, for instance, items in online platforms, search 
results, or model parameters in large language models -- is large or even infinite, it is 
common to assume that the reward function is parameterized by an unknown mapping, typically 
taken to be linear~\citep{Dueling-saha-linear, pmlr-v162-bengs22a, 
li2024feelgoodthompsonsamplingcontextual} or belonging to a reproducing kernel Hilbert 
space (RKHS) associated with a given kernel~\citep{Dueling-xu2024-kernel, 
Dueling_Stackelburg_Krause, kayal2025bayesian}. However, such parametrizations may be 
overly restrictive in practice and can limit the expressiveness of the model. To address 
this limitation, the present paper considers a more general setting in which the reward 
function belongs to a rich and generic function class, and must be learned from preference 
feedback obtained through pairwise comparisons of selected arms.




\section{Setting}
\phantomsection\label{sec:setting}

\paragraph{Notation.}
The set $\{1,\ldots,T\}$ is denoted by $[T]$. The sigmoid function is defined as $\sigma(x) := (1+e^{-x})^{-1}$, and $\dot{\sigma}$ and $\ddot{\sigma}$ denote its first and second derivatives, respectively. The log-loss function is defined as
\(
\ell(a,b) := - a \Log{b} - (1-a)\Log{1-b}.
\)
The inner product between two vectors $u, v$ is denoted $\langle u, v \rangle$. We define $\|x\| := \sqrt{\langle x, x \rangle}$ and, for any positive definite matrix $V \in \mathbb{R}^{d \times d}$, $\|x\|_V := \sqrt{\langle x, Vx \rangle}$, as the Euclidean and weighted Euclidean norms of $x$, respectively. We denote by $\mathbb{B}_2^d(r)$ the $d$-dimensional Euclidean ball of radius $r$ centered at the origin. Standard asymptotic notation $\cO(\cdot)$ and $\Omega(\cdot)$ is used throughout, while $\widetilde{\cO}(\cdot)$ suppresses logarithmic factors.

\paragraph{Problem formulation.}
We consider a stochastic dueling bandit problem defined by a decision set $\cX$ and a 
reward function class $\cF$. At each round, the learner selects a pair of arms -- referred 
to as a \emph{duel} -- from $\cX$ and observes preference feedback over the selected arms. This formulation contrasts with the 
classical stochastic bandit setting, in which the learner selects a single arm $x \in \cX$ 
at each round $t$ and observes a stochastic absolute reward $r_t(x)$. In dueling bandits, at round 
$t$, the learner selects a pair of arms $(x_t, x'_t) \in \cX^2$. Rather than observing the 
individual rewards $r_t(x_t)$ and $r_t(x'_t)$, the learner receives stochastic preference 
feedback $y_t \in \{0,1\}$, where $y_t = 1$ indicates that arm $x_t$ is preferred over 
$x'_t$, and $y_t = 0$ otherwise. Yet, the objective is to cumulate as much reward as possible. 
We impose the following realizability and boundedness conditions, common in the stochastic setting.

\begin{assumption}[Realizability]
\phantomsection\label{Asp:mean_value_function}
$\exists f^\star \in \cF,~\forall t \in [T],~\forall 
x \in \cX$,~~
\(
f^\star(x) = \mathbb{E}\!\left[r_t(x) \mid x \right].
\)
\end{assumption}

\begin{assumption}[Bounded reward class]
\phantomsection\label{Asp:bounded_reward_cost}
$\exists S > 0,~\forall x \in \cX,~\forall f \in \cF$,~~
\(
f(x) \in [-S, S].
\)
\end{assumption}

\paragraph{Observation model.}
The link between the preference feedback and the underlying rewards is assumed to follow the Bradley--Terry (BT) model~\citep{Bradely_T}, a standard modeling choice in the dueling bandits literature~\citep{Dueling-saha-linear, pmlr-v162-bengs22a, li2024feelgoodthompsonsamplingcontextual}. Under this model, the probability that arm $x_t$ is preferred over $x'_t$ is given by
\begin{align*}
\mathbb{P}\!\left(x_t \succ x'_t\right)
= \mathbb{P}\!\left(y_t = 1 \mid x_t, x'_t\right)
= \sigma\!\left(f^\star(x_t) - f^\star(x'_t)\right),
\end{align*}
where $x_t \succ x'_t$ denotes that arm $x_t$ is preferred over arm $x'_t$. While the BT model provides a tractable link between rewards and preferences, it introduces a problem-dependent constant
\[
\kappa := \sup_{x,x' \in \cX} \frac{1}{\dot{\sigma}\!\left(f^\star(x) - f^\star(x')\right)},
\]
which captures the degree of non-linearity of the sigmoid function. This quantity can be prohibitively large: it scales as $\kappa \approx e^{2S}$, so that even $S = 5$ already yields $\kappa > 2.2 \times 10^{4}$. Consequently, it requires careful treatment in both the algorithm design and the theoretical analysis in order to avoid a detrimental impact on the resulting regret guarantees.

\paragraph{Performance measure.}\label{sec:Metrics_Discussion}
After selecting a pair of arms $(x_t, x'_t)$ at round $t$, the learner incurs an 
instantaneous regret. Two standard notions of instantaneous regret in the dueling bandits 
literature~\citep{Dueling-saha-linear, pmlr-v162-bengs22a, li2024feelgoodthompsonsamplingcontextual} 
are the \emph{average} and \emph{weak} instantaneous regrets, defined respectively as
\(
\rho_t^{a} := f^\star(x^\star) - \frac{f^\star(x_t) + f^\star(x'_t)}{2}
\; \text{and} \;
\rho_t^{w} := f^\star(x^\star) - \max\bigl\{f^\star(x_t),\, f^\star(x'_t)\bigr\},
\)
where $x^\star := \arg\max_{x \in \cX} f^\star(x)$ denotes the optimal arm. After $T$ 
rounds, the cumulative regret of a policy is defined as
\(
\cR_T^\tau := \sum_{t=1}^T \rho_t^\tau, \; \tau \in \{a, w\}.
\)
A desirable policy achieves sublinear regret, i.e.,
\(
\lim_{T \to \infty} \frac{\cR_T^\tau}{T} = 0,
\)
which implies that the policy asymptotically identifies the optimal arm and selects it 
consistently in comparisons. By definition, $\cR_T^w \leq \cR_T^a$. Therefore, throughout 
the remainder of the paper, we restrict attention to the average cumulative regret
\[
\cR_T := \cR_T^a = \sum_{t=1}^T \left( f^\star(x^\star) - \frac{f^\star(x_t) + f^\star(x'_t)}{2} \right).
\]





\section{Related Work and Contributions}

\paragraph{Dueling bandits.}

Learning with indirect feedback was first studied in supervised preference 
learning~\citep{NIPS2004_5b168fdb, 10.1145/1102351.1102369}, and later extended to online 
and sequential settings motivated by applications such as human-provided 
feedback~\citep{Dueling_Yue_old, Dueling_Yue, Dueling_classification}.  We refer the readers to~\cite{Dueling_Survey} for a detailed survey on
various works covering different settings of dueling bandits.


An initial line of work considers finite multi-armed domains and learns a preference matrix 
specifying pairwise relations among arms. Such approaches often rely on efficient sorting or 
tournament schemes based on win frequencies per 
action~\citep{10.5555/2986459.2986709, Dueling_zoghi2, pmlr-v37-komiyama15, NIPS2016_9de6d14f, Dueling_falahatgar17a}. 
Rather than selecting both arms jointly, these strategies simplify the problem by 
choosing one arm at random~\citep{Dueling-zoghi14, Dueling_factored}, 
greedily~\citep{Dueling-chen17c}, or from the set of previously selected 
arms~\citep{Dueling_ailon14}. Building on an online square-loss oracle, \cite{Dueling-saha22a} 
propose an efficient and optimal algorithm when the preference matrix is 
well-specified by a given function class.

To handle large or even infinite decision spaces, an alternative paradigm has emerged, namely 
\emph{utility-based dueling bandits}, to which the present work belongs. A prominent line of research adopts linear utility models. \cite{Dueling-saha-linear} and 
\cite{pmlr-v162-bengs22a} both employ UCB-style estimators but with different decision 
schemes: the former constructs a set of plausible winners and selects the most informative 
pair among them, while the latter greedily selects the first arm and then chooses a 
competitive second arm to duel. Both achieve a regret of $\widetilde{\cO}(\kappa d \sqrt{T})$, 
where $d$ is the dimension of the arm space. Similarly, \cite{di2024varianceaware} present a 
variance-aware algorithm that proceeds in phases, maintaining an active set of arms via a 
UCB-style criterion, and achieves
\(
\widetilde{\cO}\!\left(d\kappa \sqrt{\sum_{t=1}^{T} \eta_t^2} + d\kappa^2\right),
\)
where $\{\eta_t\}$ denotes the noise variance sequence. Alternatively, \cite{li2024feelgoodthompsonsamplingcontextual} propose a Thompson sampling 
method that also achieves $\widetilde{\cO}(\kappa d \sqrt{T})$. However, such linear methods 
have limited practical appeal - as they cannot capture the complex nonlinear utility functions 
arising in real-world problems - motivating the extension to Reproducing kernel Hilbert spaces (RKHS). Considering a kernelized 
logistic negative log-likelihood loss to estimate the utility, but with different decision 
rules, \cite{Dueling-xu2024-kernel} achieve a regret of $\widetilde{\cO}\bigl((\gamma_T 
T)^{3/4}\bigr)$ -- where $\gamma_T$ denotes the maximum information gain -- by retaining 
one arm from the previous round and selecting the second as the maximizer of a UCB on the 
pairwise preference. \cite{Dueling_Stackelburg_Krause} adopt a game-theoretic arm-selection 
strategy and achieve a regret of $\widetilde{\cO}(\gamma_T \kappa^2 \sqrt{T})$. Building on a more involved phased approach, \cite{kayal2025bayesian} removed the dependence on $\kappa$ compared to~\cite{Dueling_Stackelburg_Krause}. Finally, the recent work of \cite{verma2025neural} considers 
neural dueling bandits in the wide-network regime (NTK), leveraging neural tangent features 
for preference prediction to achieve a regret of
\(
\widetilde{\cO}\!\big(\kappa\tilde{d}\sqrt{T}\big),
\) with $\tilde{d}$ the effective dimension.

It is worth noting that preference-based feedback has played a central role in integrating 
human feedback into learning systems, yielding notable successes across a range of 
applications~\citep{Duleing-stiennon2022-learningsummarizehumanfeedback, Dueling-assistant, 
Dueling-saha23a-RL, Dueling-zhu23f-RL, Dueling-RLHF, Dueling-munos24a-RL}.

\paragraph{Bradley-Terry preference model and logistic bandits.}
The Bradley-Terry (BT) model is a seminal framework to derive binary preferences from absolute scores: for items with scores $v_1$ and $v_2$, the preference probability is $\sigma(v_1-v_2)$ where $\sigma$ is the sigmoid function $x \mapsto (1+e^{-x})^{-1}$. Consequently, dueling bandits under the BT assumption inherit the technical challenges of logistic bandits. Early logistic bandit analyses relied on global linearization~\citep{Filippi}, yielding regret bounds of $\cO(\kappa\sqrt{T})$, where $\kappa$ can scale exponentially with the ambient dimension. In the linear utility setting, recent advances~\cite{Faury20a, pmlr-v130-abeille21a, faury22} demonstrate that a careful treatment of logistic non-linearity leads to an improved regret bound of $\cO(\sqrt{T/\kappa})$. Whether such improvements are intrinsically tied to the linear structure of the utility function, or can be extended to more general classes of reward functions, remains an open question.
\paragraph{Eluder Dimension.}
The key challenge in going beyond linear and kernelized mappings is the need to design 
sets of plausible models compatible with observed data -- that is, to establish 
concentration guarantees -- and to quantify how their width translates into online 
prediction error. The seminal work of~\citet{Eluder_russo} addresses both challenges, relating the former to covering numbers and the latter to a newly introduced complexity measure---the \emph{eluder dimension}---which is further studied and characterized in~\cite{li2022understanding}.This notion has been employed by 
\cite{sekhari2023contextual} in the context of preference-based bandits. However, it is 
primarily developed for additive noise models (e.g., building on least-squares regression), 
which precludes extending logistic bandits to generic function classes while preserving the 
sharp non-linearity treatment introduced by \cite{Faury20a}. Recently,~\citet{Local2025eluder} attempted to reconcile both approaches by introducing a \emph{localized eluder dimension} defined for arbitrary loss functions (in line with~\cite{liu-eluder-proof}, as opposed to the squared loss). They further show---independently of the loss function---that a global eluder dimension necessarily entails an unfavorable dependence on $\kappa$. This observation motivates their localization approach, evaluating the eluder dimension only over a suitable subfamily $\cF' \subset \cF$. The localized eluder dimension is independent of $\kappa$ but applies only when the confidence set is small enough, yielding an additional regret term quantifying the regret incurred before localization. Unfortunately, this regret term is uncontrolled in general and when instantiated to generalized linear models turns out to be linear in $T$, rendering the approach vacuous (see~\cref{sec:Closer_look_Loc_Elud} for details).

\paragraph{Main contributions.}
This work is the first to study the complexity of the preference bandit setting with general 
utility classes, introducing a novel complexity measure tailored to yield fine-grained 
regret guarantees. Our contributions may be summarized as follows.

\begin{enumerate}
\item \textbf{Locally sensitive eluder dimension.}
The lower bound established in~\cite[Theorem~2]{Local2025eluder} shows that the global eluder dimension of~\cite{Eluder_russo} style necessarily scales with $\kappa$ in the generalized linear setting, leading to an undesirable dependence that motivates the introduction of a refined complexity measure. We therefore introduce a new notion of eluder dimension (\cref{def:defsigmaeluder}), tailored to the preference feedback setting. This quantity properly accounts for the non-linearity induced by the logistic loss, and thereby avoids any unfavorable dependence on $\kappa$.

\item \textbf{GINOP algorithm (\cref{sec:algo}).} We propose a Generic INformed-OPtimistic 
strategy for preference bandits with general reward function classes. The algorithm takes 
the reward class as input and iteratively constructs adequate confidence sets, which are 
then leveraged to select actions $(x_t, x'_t)$ in a correlated manner, balancing the dual 
objectives of optimistically collecting reward and enforcing dissimilarity to gather more 
informative comparisons.
    
\item \textbf{Regret guarantees (\cref{thm:Alg_perf}).} Through a careful analysis, we 
demonstrate that, over a time horizon $T$, \textbf{GINOP} achieves a first-order regret 
bound that is, up to logarithmic factors, of order
\[
\sqrt{
    \underbrace{\delud(\Delta\cF,\, \Delta_{f^\star},\, 1/T^2)}_{\text{locally sensitive eluder dimension}}
    \;\underbrace{\Log{\cN_T}}_{\text{log-covering number}}
    \;\frac{T}{\dot{\sigma}^\star}
}
+ \kappa\,\Gamma_T,
\]
where $\dot{\sigma}^\star$ denotes the local curvature of the sigmoid at the optimal 
choice and is a well-behaved constant, and $\Gamma_T$ is a lower-order instance-dependent 
term. We further instantiate our results to the linear and kernelized reward function classes, 
and show that our eluder dimension introduces no additional $\kappa$ dependence, thereby 
demonstrating that our approach surpasses previous work in both generality and sharpness 
(\cref{sec:Instantiation}). From learning-theoretic standpoint, \cref{thm:Alg_perf} shows that learning from preference feedback is as statistically efficient as learning from direct reward observations.

\item \textbf{Empirical evaluation (\cref{sec:Experiments}).} We corroborate our theoretical findings with empirical 
evaluations, comparing \textbf{GINOP} against several existing baselines.
\end{enumerate}

\section{Learning Process}
\phantomsection\label{sec:Learning_and_Estimation}

From the history of observations $\cH_t = \{x_1, x'_1, y_1, \ldots, x_t, x'_t, y_t\}$, 
the learner constructs estimators of the underlying reward function. Since rewards are only revealed through preferences, standard least-squares estimation is not applicable,  precluding a direct 
use of~\cite{Eluder_russo}'s results. We instead leverage a procedure based on the log-loss, which yields refined guarantees.

Under the BT model, the preference feedback depends on the reward difference 
$f^\star(x) - f^\star(x')$ associated with a pair $(x, x')$, where $f^\star \in \cF$. 
We introduce the \emph{difference operator} $\Delta_f$, the induced function class 
$\Delta\cF$, and the corresponding excess loss class $\Phi(\Delta{\cF})$:
\begin{align*}
    \Delta_f &: (x,x') \in \cX^2 \mapsto f(x) - f(x') \in [-2S,\, 2S], \qquad \Delta\cF := \bigl\{ \Delta_f : f \in \cF \bigr\},\\
    \Phi(\Delta{\cF}) &:= \Bigl\{ (y,x,x') \mapsto 
        \ell\bigl(y,\, \sigma(\Delta_f(x,x'))\bigr) 
        - \ell\bigl(y,\, \sigma(\Delta_{f^\star}(x,x'))\bigr) 
        : f \in \cF \Bigr\}.
\end{align*}

Equipped with these definitions, the true reward function $f^\star$ is estimated from 
$\cH_{t-1}$ via maximum likelihood estimation (MLE). For all $t \in [T]$,
\begin{equation}
\label{eq:f_hat}
\hat{f}_t \in \argmin_{f \in \cF} \cL(f;\, \cH_{t-1}),
\end{equation}
where the negative log-likelihood is defined as
\(
\cL(f;\, \cH_{t-1})
:= \sum_{s=1}^{t-1} \ell\Bigl(y_s,\, \sigma\!\bigl(\Delta_f(x_s, x'_s)\bigr)\Bigr).
\)

We build confidence sets $\cF_t \subset \cF$ consisting of functions close to $\hat{f}_t$ under the log-loss. For all $t \geq 1$,
\begin{equation}
\label{eq:conf.sets}
\begin{aligned}
\cF_t &= \left\{ f \in \cF : \cL(f; \cH_{t-1}) \le \cL(\hat{f}_t; \cH_{t-1}) + \beta_t(\delta) \right\}, \\
\text{where} \quad
\beta_t(\delta) &= \frac{5}{2} + 60(2S+1)\,\log\!\Biggl(\frac{\cN_T(\Phi(\Delta{\cF}))\bigl(e+\log(1+t)\bigr)}{\delta}\Biggr)= \cO\!\bigl(\log\frac{\cN_T(\Phi(\Delta{\cF}))}{\delta}\bigr)
\end{aligned}
\end{equation}

The sequence $\{\beta_t(\delta)\}_t$ is non-decreasing and defines the confidence radius. It captures the complexity of the problem through $\cN_T(\Phi(\Delta{\cF}))$, the $1/T$-covering number of the class $\Phi(\Delta{\cF})$ under the uniform metric. Concentration results of the form~\eqref{eq:conf.sets} hold naturally for finite function classes, with confidence widths of order $\beta_t(\delta) = \cO(\log(\lvert \cF \rvert / \delta))$. Using the covering number allows these results to be extended to infinite function classes. \cref{lem:Ft_includes_r} - adapted from~\cite[Proposition~22]{Local2025eluder} - ensures $\{\cF_t\}_{t\geq1}$ is a valid sequence of confidence sets. 

\begin{restatable}[]{lemma}{lemFtIncludesr}
\phantomsection\label{lem:Ft_includes_r}
Let $\cF_t$ be defined as in~\cref{eq:conf.sets} for all $t \geq 1$. Under \cref{Asp:bounded_reward_cost}, with probability at least $1 - \delta$, it holds that
\(
f^\star \in \bigcap_{t \geq 1} \cF_t.
\)
\end{restatable}

Finally, while $\cF_t$ contains models that are compatible with the observations up to time $t$ (small empirical risk), we are also interested in the prediction error associated with $\cF_t$. We define the worst-case uncertainty in the performance gap between two arms $x,x'$ by
\begin{align}\label{eq:omega_def}
\omega_{t}(x,x')
= \sup\nolimits_{f,f' \in \cF_t}
\bigl( \Delta_f(x,x') - \Delta_{f'}(x,x') \bigr).
\end{align}

\section{Algorithm and Main Results}\label{sec:algo}

{    
    \renewcommand{\thealgocf}{}
        \begin{algorithm2e}[h]
            \SetAlgorithmName{GINOP}{}{}\label{alg:sinop}
            \caption{\underline{G}eneric \underline{IN}formative \underline{OP}timism for Preference Bandits}
            \SetAlgoLined
            \DontPrintSemicolon
            \textbf{Inputs:} $\cX$, $\cF$, $\{\beta_t\}$, $\delta= \frac{1}{T^2}$, $\mathcal{H}_0 = \emptyset$. \\
            \For{$t = 1, \ldots, T$}{
                Construct $\hat{f}_{t}$, $\omega_t$ from~\eqref{eq:f_hat}, ~\eqref{eq:conf.sets} and ~\eqref{eq:omega_def}. \label{step:ERM}\;
            Play\label{step:Optimal_arms}:\,$x_t, x_t^\prime  = \argmax\limits_{x,x' \in \cX} \,   \hat{f}_{t}(x)+  \hat{f}_{t}(x^\prime) + \omega_t(x,x').$
                
                Observe feedback $y_t$.\;
                Update history:
                $\mathcal{H}_t = \mathcal{H}_{t-1} \cup \{ x_t, x_t^\prime, y_t\}$.\;
            }
        \end{algorithm2e}

}

We introduce the \textbf{GINOP} algorithm for preference bandits with a general reward function class. At each round $t$, the algorithm leverages the available historical data to compute the MLE $\hat{f}_t$ together with an associated uncertainty measure $\omega_t$ (Step~\ref{step:ERM}). It then selects a pair of arms $(x_t, x'_t)$ that maximizes an \emph{informative optimism} criterion (Step~\ref{step:Optimal_arms}), which augments the estimated cumulative reward of the pair, $\hat{f}_t(x) + \hat{f}_t(x')$, with a positive bonus $\omega_t(x,x')$ that explicitly quantifies the uncertainty in the \emph{reward difference} between the two arms, thereby promoting targeted exploration. Crucially, since $\omega_t(x,x) = 0$, the algorithm is structurally discouraged from selecting identical arms, particularly in the early rounds. In contrast to standard bandit settings with direct feedback, in which the exploration bonus typically aggregates the uncertainties of individual arms, the bonus employed here is specifically tailored to the dueling structure by prioritizing comparisons between \emph{distinct} arms. A distinguishing feature of \textbf{GINOP} lies in its decision rule, which jointly selects both dueling arms in a single optimization step. This stands in contrast to prior approaches, which typically proceed in two stages: either (i) following a \emph{leader--follower} paradigm, in which a leader arm is selected first and the follower is chosen conditionally on the leader~\citep{pmlr-v162-bengs22a,li2024feelgoodthompsonsamplingcontextual,Dueling-xu2024-kernel,Dueling_Stackelburg_Krause,verma2025neural}; or (ii) by first constructing, at each round, a set of \emph{plausible winners} and subsequently selecting the dueling pair from within this set~\citep{Dueling-saha-linear,di2024varianceaware,kayal2025bayesian}.

\paragraph{Locally sensitive eluder dimension.}
As demonstrated by~\cite{Eluder_russo}, traditional complexity measures for function classes -- such as the VC dimension and covering numbers-- do not suffice to analyze bandit settings. To address this limitation, the authors introduce the \emph{eluder dimension}, which quantifies how the \emph{learning error} relates to the \emph{on-policy prediction error} - informally, how $w_t(x_t,x'_t)$ reduces as $\cF_t$ shrinks. However, their 
definition is tailored to additive noise models and does not correctly accommodate 
Bernoulli preference feedback. To overcome this limitation, we introduce the 
\emph{locally sensitive eluder dimension}, a novel complexity notion better suited to the 
preference bandit setting.



\begin{definition}[$(\sigma, \varepsilon)$-eluder dimension]\label{def:defsigmaeluder}
    Let \(\mathcal{Z}\) be a set and $\cQ = \{ q: \cZ \rightarrow \mathbb{R} \}$. Fix $q^\star \in \cQ$ and let $\varepsilon>0$. Define $\bar{\varphi}$ the excess log-loss as 
    $\bar{\varphi}(a,b) = \ell(a,b) - \ell(a,a)$.
    \begin{enumerate}
        \item Given a sequence \(\bar{z} =(z_{1},z_{2},\ldots,z_{n}) \in \cZ^n\), we say that $z\in\cZ$ is $(\sigma, \varepsilon)$-independent from $\bar{z}$ w.r.t. $(\cQ,q^\star)$ if there exist $q\in\cQ$ s.t.
        \begin{align}
            &\sum_{j\leq n} \bar{\varphi}\big(\sigma(q^\star(z_j)),\sigma(q(z_j))\big) \leq \varepsilon^2 \label{eq:eluder_fit}\\
            \text{and} &\nonumber\\
            &\dot{\sigma}(q^\star(z))\left(q^\star(z) - q(z)\right)^2 > \varepsilon^2 \label{eq:eluder_pred}
        \end{align}
        \item The $(\sigma, \varepsilon)$-eluder dimension of $\cQ$ localised at $q^\star$, denoted $\delud(\cQ,q^\star,\varepsilon)$, is the length $d$ of the longest sequence $(z_1, z_2, \ldots,z_d)\in \cZ^d$ for which $\exists \varepsilon'>\varepsilon$ s.t. for any $i\leq d$, $z_i$ is $(\sigma, \varepsilon)$-independent from $(z_1,z_2,\ldots,z_{i-1})$ w.r.t. $(\cQ,q^\star)$.
    \end{enumerate}
\end{definition}

The notion of $\varepsilon$-independence used in~\cref{def:defsigmaeluder} follows the same rational as~\cite{Eluder_russo}: $z$ (resp. $\bar{z})$ has a large (resp. small) loss and hence are independent. Note that losses are considered between $q$ and $q^\star$ rather than on pairs $(q,q')$. This is a minor variation that allows to fix a nominal model of interest within $\cQ$ (see~\cite{li2022understanding} for a more thorough discussion). The main difference lies in the fact that \cite{Eluder_russo} defines the eluder dimension 
via the square loss, which limits its applicability to additive noise models, whereas -- 
in line with~\cite{Local2025eluder} -- we use $\bar{\varphi}$ for the learning error (in \cref{eq:eluder_fit}) to 
obtain an eluder dimension tailored to the logistic structure of our problem. However, we 
depart from~\cite{Local2025eluder} by leveraging a distinct prediction error measure. This translates in \cref{eq:eluder_pred} using $\dot{\sigma}(q^\star(z))\left(q^\star(z) - q(z)\right)^2$ instead of $\bar{\varphi}\big(\sigma(q^\star(z)),\sigma(q(z))\big)$. As a result, local sensitivity is taken into account directly in the eluder definition while preserving the original ideas as 
$$\bar{\varphi}\big(\sigma(q^\star(z)),\sigma(q(z))\big) \simeq \dot{\sigma}(q^\star(z))\left(q^\star(z) - q(z)\right)^2.$$
It is worth noting that this quantity appears only in the analysis and is not required for the implementation of the algorithm.

\paragraph{Theoretical guarantees.}
We are now ready to state the main result, which provides regret guarantees for the 
\textbf{GINOP} algorithm over a general reward function class. A high-level overview of 
the proof is provided in the subsequent paragraph, while the complete proof is deferred 
to~\cref{sec:Alg_Performance_Proof}.

\begin{restatable}[]{theorem}{ThmAlgproof}
\phantomsection\label{thm:Alg_perf}
Under~\cref{Asp:mean_value_function} and~\cref{Asp:bounded_reward_cost}, letting 
$\varepsilon = 1/T^2$ and $\delta = 1/T^2$, the \textbf{GINOP} algorithm satisfies
\begin{align*}
\cR_T
&=
\widetilde{\cO}\!\left(
\sqrt{\frac{T}{\dot{\sigma}^\star}
\Log{\cN_T(\Phi(\Delta \cF)) / \delta}\,
\delud(\Delta\cF, \Delta_{f^\star},\varepsilon)}
+ \kappa\, \Gamma_T \right),\\[4pt]
\text{where} \quad \Gamma_T
&=
\widetilde{\cO}\!\Bigl(
\log\bigl(\cN_T(\Phi(\Delta \cF))/\delta\bigr)\,
\delud^2(\Delta\cF,\Delta_{f^\star},\varepsilon)\Bigr),
\quad
\dot{\sigma}^\star
:=
\dot{\sigma}\!\left(\Delta_{f^\star}(x^\star,x^\star)\right)=\frac{1}{4}.
\end{align*}
\end{restatable}

\cref{thm:Alg_perf} establishes a first-order regret upper bound. The leading $\sqrt{T}$ term depends on three quantities: (i) $\dot{\sigma}^\star$, the 
local curvature of the sigmoid at the optimal arms, which is a well-behaved constant in 
contrast to $\kappa$; (ii) the log-covering number, which controls statistical overfitting 
and is a standard feature of complexity notions in statistical learning theory; and (iii) 
the locally sensitive eluder dimension associated to $\Delta_\cF$, which captures how 
effectively the value of unobserved actions can be inferred from observed samples. The dependence on $\kappa$ is confined to the lower-order term $\kappa\,\Gamma_T$, where 
$\Gamma_T$ is a problem-dependent quantity. That the leading-order term depends on the local (small) curvature $\dot{\sigma}^\star$ rather than on the global one $\kappa$ assesses that, from a learning-theoretic perspective, learning from preference feedback is statistically as efficient as learning from direct reward observations. While assessing the tightness of the upper 
bound is difficult in the general case, we instantiate our results to specific reward 
function classes in~\cref{sec:Instantiation}, enabling a more refined evaluation.

\paragraph{Idea of proof.}\label{sec:proofsketch}
For convenience, we denote by $\cG\cE$ the intersection of all high-probability events required to establish our results (see~\cref{remark:GE}). Complete proofs are deferred to~\cref{sec:concentrations} and \cref{sec:app-perf}.
The algorithm relies on a tailored notion of optimism in its decision rule (Step~\ref{step:Optimal_arms}). In contrast to standard bandit settings with direct feedback, where the exploration bonus typically aggregates uncertainties over individual arms, the bonus here is specifically designed to capture the structure of dueling bandits by prioritizing comparisons between \emph{distinct} arms. Although this departs from the classical formulation, the proposed informative optimism  preserves the underlying principle of optimism by enabling tight round-wise regret control (see~\cref{lem:roundwise-error}, proved in~\cref{sec:round_wise_regret}).

\begin{restatable}[Round-Wise Regret]{lemma}{LemRoundwiseError}
\phantomsection\label{lem:roundwise-error}
Under~\cref{Asp:mean_value_function} and~\cref{Asp:bounded_reward_cost}, on the event $\cG\cE$, the following holds:
\begin{align*}
\forall t \in [T], \;
\rho(t) := f^\star(x^\star)
- \frac{f^\star(x_t) + f^\star(x'_t)}{2}
\;\leq\;
\omega_t(x_t,x'_t).
\end{align*}
\end{restatable}


The regret bound then hinges on efficiently upper bounding the on-policy prediction error $\omega_t(x_t,x'_t)$ accumulated along the trajectory, which we establish in \cref{lem:sigma._omega_upperbound}. This result, proved in~\cref{sec:sigma._omega_upperbound_proof}, constitutes one of our principal theoretical contributions as it provides a locally sensitive characterization of the on-policy prediction error through the introduced eluder dimension.

\begin{restatable}[]{lemma}{lemSigmaOmegaproof}
\phantomsection\label{lem:sigma._omega_upperbound}
On the event $\cG\cE$, the following hold:
{
\begin{align*}
&\sum_{t=1}^T \sqrt{\dot{\sigma}\!\left(\Delta_{f^\star}(x_t,x'_t)\right)} \,\omega_t(x_t,x'_t)=
\widetilde{\cO}\!\left(
\sqrt{T\,\log(\cN_T(\Phi(\Delta \cF))/\delta)\, \delud(\Delta\cF, \Delta_{f^\star}, \varepsilon)}
\right), \\
&\sum_{t=1}^T \omega_t(x_t,x'_t)^2
= \widetilde{\cO}\Bigl(\kappa
\log(\cN_T(\Phi(\Delta \cF))/\delta)\,
\delud^2(\Delta\cF,\Delta_{f^\star},\varepsilon)\Bigr).
\end{align*}
}
\end{restatable}

Finally, we carefully decompose the regret so that it scales with the sensitivity of the sigmoid function around the optimal action $(x^\star,x^\star)$.
{
\begin{align*}
\cR_T
&\le
\frac{1}{\sqrt{\dot{\sigma}^\star}}
\sum_{t=1}^T
\sqrt{\dot{\sigma}\!\left(\Delta_{f^\star}(x_t,x'_t)\right)}
\,\omega_t(x_t,x'_t)
+\frac{1}{\sqrt{2\dot{\sigma}^\star}}
\sum_{t=1}^T
\sqrt{\rho(t)}
\,\omega_t(x_t,x'_t) \\
&\le
\underbrace{
\frac{1}{\sqrt{\dot{\sigma}^\star}}
\sum_{t=1}^T
\sqrt{\dot{\sigma}\!\left(\Delta_{f^\star}(x_t,x'_t)\right)}
\,\omega_t(x_t,x'_t)
}_{(a)}
+\underbrace{
\frac{1}{\sqrt{2\dot{\sigma}^\star}}
\sqrt{\sum_{t=1}^T \omega_t(x_t,x'_t)^2}
}_{(b)}
\sqrt{\cR_T},
\end{align*}
}where the second inequality follows from Taylor expansion and Cauchy--Schwarz inequality. This yields a quadratic inequality in $\cR_T$, whose solution satisfies
\(
\cR_T \le a + 2b^2.
\)
Bounding the terms $(a)$ and $(b)$ using ~\cref{lem:sigma._omega_upperbound} completes the proof and yields the stated regret bound.

\section{Instantiation to Familiar Utility Classes}
\phantomsection\label{sec:Instantiation}
In \cref{sec:algo}, the \textbf{GINOP} algorithm and its first-order regret upper bound 
were established for a general reward function class. To further illuminate the behavior 
of the algorithm and assess the sharpness of the theoretical guarantees in canonical 
settings, we now specialize the analysis to two classical utility models: the linear and 
kernelized reward function classes.

\paragraph{Linear utility.}\label{sec:linear_instance}
In this setting, $\cX \subset \mathbb{B}_2^d(1)$ and, for any $x \in \cX$, the reward 
function is given by
\(
f^\star(x) = \langle \theta^\star, x \rangle,
\)
where $\theta^\star \in \mathbb{R}^d$ is an unknown parameter satisfying 
$\|\theta^\star\| \leq S$. Accordingly, the class of linear reward functions $\cF_{lin}$ 
can be identified with the parameter set $\Theta \subset \mathbb{B}_2^d(S)$, i.e.,
\(
\cF_{lin} = \left\{ \langle \theta, \cdot \rangle : \theta \in \Theta \right\}.
\)
Under the Bradley--Terry model, the preference feedback satisfies
\(
y_t \sim \mathrm{Bernoulli}\!\left(\sigma\!\left(\langle \theta^\star, x_t - x'_t \rangle\right)\right).
\)
The algorithmic implementation of \textbf{GINOP} in the linear utility setting, including 
the characterization of the MLE $\hat{f}_t$ and the uncertainty quantifier $\omega_t$, is 
deferred to~\cref{sec:Eluder_for_linear}.
To specialize the regret bound of \cref{thm:Alg_perf} to $\cF_{lin}$, it remains to 
control the key complexity quantities appearing in the theorem, namely 
$\Log{\cN_T(\Phi(\Delta\cF))}$ and $\delud(\Delta\cF, \Delta_{f^\star}, \varepsilon)$, 
in the case $\cF = \cF_{lin}$.

\begin{restatable}[]{proposition}{propsigmaeluderlinear}
\phantomsection\label{prp:propsigmaeluderlinear}
In the linear reward setting, the following bounds hold:
\[
\forall\, \varepsilon > 0, \quad
\delud(\Delta{\cF_{lin}}, \Delta_{f^\star}, \varepsilon)
= \cO\bigl(d \log(1 + 1/\varepsilon)\bigr)
\quad \text{and} \quad
\Log{\cN_T(\Phi(\Delta\cF_{lin}))}
= \cO\bigl(d \log T\bigr).
\]
\end{restatable}

Given \cref{prp:propsigmaeluderlinear}, whose proof is deferred to 
Appendix~\ref{sec:Eluder_for_linear}, the following corollary is immediate.

\begin{restatable}[]{corollary}{CorAlgproof}
\phantomsection\label{Cor:Alg_perf_lin}
In the linear reward setting, letting $\varepsilon = 1/T^2$ and $\delta = 1/T^2$, the 
\textbf{GINOP} algorithm satisfies
\(
\cR_T
=
\widetilde{\cO}\!\left(
d \sqrt{\frac{T}{\dot{\sigma}^\star}}
+ \kappa d^3 \right).
\)
\end{restatable}

In the linear utility setting, similarly to the standard logistic bandit framework (i.e., without preference feedback) studied in~\cite{pmlr-v130-abeille21a}, {\bf (i)} the dependence on $\kappa$ is confined to lower-order terms, and {\bf (ii)} the leading term depends on the local curvature of the sigmoid at the optimal action, $\dot{\sigma}^\star$. In particular, \cite{pmlr-v130-abeille21a} establish a minimax regret bound of the form $\widetilde{\cO}\bigl(d \sqrt{\dot{\sigma}^\star T} + \kappa\bigr)$, which we match up to a factor $1/\dot{\sigma}^\star$. We emphasize that this discrepancy is not a weakness of our analysis, but rather stems from a difference in the reward metric: our objective involves $f^\star(x)$, whereas their formulation considers $\sigma(f^\star(x))$. A straightforward adaptation of their lower bound techniques to our setting confirms the tightness of~\cref{Cor:Alg_perf_lin} with respect to $\dot{\sigma}^\star$.

The refined characterization of the scaling with respect to the sigmoid curvature allows one to fully exploit a key feature of preference feedback: actions are selected as pairs, and the observed signal depends on relative rewards. For the optimal pair $(x^\star, x^\star)$, the local curvature of the sigmoid satisfies $\dot{\sigma}^\star = \dot{\sigma}(0) = 1/4$. Consequently, the dependence on curvature disappears from the leading-order term, yielding performance comparable to that obtained under direct reward observations. In particular, the resulting guarantee aligns with that obtained from~\cite[Proposition~4 + Example~4]{Eluder_russo} in a direct reward model.

Consequently, for dueling bandits with linear utilities, \textbf{GINOP} is minimax optimal: its regret bound matches the lower bound $\Omega(d \sqrt{T})$ established by~\cite{li2024feelgoodthompsonsamplingcontextual}. This stands in contrast to prior work~\citep{Dueling-saha-linear,pmlr-v162-bengs22a,li2024feelgoodthompsonsamplingcontextual}, whose leading term carries an additional dependence on $\kappa$.


\paragraph{Kernelized utility.}
Following~\citep{Dueling_Stackelburg_Krause, Dueling-xu2024-kernel, kayal2025bayesian}, 
we assume that the utility function $f^\star$ belongs to a known reproducing kernel Hilbert 
space (RKHS). Let $k : \cX \times \cX \to \lR$ be a positive definite kernel, and let 
$\cH_k$ denote the associated RKHS, equipped with inner product $\langle \cdot, \cdot 
\rangle_{\cH_k}$ and norm $\|\cdot\|_{\cH_k}$. We assume $\|f^\star\|_{\cH_k} \leq S$ 
and $k(x,x) \leq 1$ for all $x \in \cX$. By the reproducing property, for all $f \in 
\cH_k$ and $x \in \cX$,
\(
\langle f,\, k(\cdot, x) \rangle_{\cH_k} = f(x).
\)
By Mercer's theorem, under mild regularity conditions, the kernel admits the spectral 
representation
\(
k(x, x') = \sum_{m=1}^{\infty} \gamma_m \phi_m(x)\phi_m(x'),
\)
where $\gamma_m > 0$ and $\{\psi_m := \sqrt{\gamma_m}\phi_m\}_{m \geq 1}$ forms an 
orthonormal basis of $\cH_k$. In particular, any $f \in \cH_k$ can be expressed as
\(
f(\cdot) = \sum_{m=1}^{\infty} \theta_m \psi_m(\cdot) 
= \langle \theta, \Psi(\cdot) \rangle,
\quad \text{with} \quad
\|f\|_{\cH_k}^2 = \|\theta\|_2^2 \leq S^2.
\)
We refer to $\{\gamma_m\}$ and $\{\phi_m\}$ as the Mercer eigenvalues and eigenfunctions 
of $k$, respectively. Accordingly, the class of kernelized reward functions $\cF_k$ can 
be identified with the parameter set $\Theta_k \subset \mathbb{B}_2^{\lN}(S)$, i.e.,
\(
\cF_k = \left\{ \langle \theta, \Psi(\cdot) \rangle : \theta \in \Theta_k \right\},
\)
with $f^\star(\cdot) = \langle \theta^\star, \Psi(\cdot) \rangle$ and $\theta^\star \in 
\Theta_k$. Under the Bradley--Terry model, the preference feedback satisfies
\(
y_t \sim \mathrm{Bernoulli}\!\left(\sigma\!\left(
\langle \theta^\star, \Psi(x_t) - \Psi(x'_t) \rangle\right)\right).
\) Define $z = (x, x') \in \cX \times \cX$. Following~\cite{Dueling_Stackelburg_Krause, 
kayal2025bayesian}, we introduce the \emph{dueling kernel}
\(
\mathds{k}(z_1, z_2) = k(x_1, x_2) + k(x_1', x_2') - k(x_1, x_2') - k(x_1', x_2),
\)
for $z_1 = (x_1, x_1')$ and $z_2 = (x_2, x_2')$. This construction satisfies 
$\|\Delta_f\|_{\cH_{\mathds{k}}} = \|f\|_{\cH_k}$, as established 
in~\cite[Proposition~4]{Dueling_Stackelburg_Krause}. The algorithmic implementation of \textbf{GINOP} in the kernelized utility setting, including 
the characterization of the MLE $\hat{f}_t$ and the uncertainty quantifier $\omega_t$, is 
deferred to~\cref{sec:Eluder_for_Ker}. We assess whether the resulting locally sensitive eluder dimension induces an undesirable 
dependence on~$\kappa$. For any $\lambda > 0$ and $T \geq 1$, the maximum information 
gain after $T$ observations is defined as
\(
\gamma_T(\lambda;\, \cX \times \cX)\)
\(=
\max\limits_{(x_1,x'_1),\dots,(x_T,x'_T) \in \cX \times \cX}
\frac{1}{2} \Log{\operatorname{det} \Bigl(I + \lambda^{-1} \mathds{K}_T \Bigr)} \),
with
\(\mathds{K}_{T}
=
\bigl[\mathds{k}\bigl((x_i,x_i'),(x_j,x_j')\bigr)\bigr]_{i,j=1}^{T}.
\)
The quantity $\gamma_T$ admits a natural geometric interpretation: it measures the logarithm of the maximum volume of the ellipsoid generated by $T$ points in $\cX \times \cX$, thereby capturing the intrinsic geometric complexity of the domain. \cref{prp:propsigmaeluderker}, whose proof is deferred to Appendix~\ref{sec:Eluder_for_Ker}, establishes a precise relationship between our eluder dimension and the maximum information gain.
\begin{restatable}[]{proposition}{propsigmaeluderkern}
\phantomsection\label{prp:propsigmaeluderker}
In the kernelized setting, the following bound holds:
\[
\forall\, \varepsilon > 0, \qquad
\delud\!\left(\Delta_{\cF_k},\, \Delta_{f^\star},\, \varepsilon\right)
= \cO\!\Bigl(\gamma_T(\lambda;\, \cX \times \cX)\Bigr),
\quad \text{with } \lambda = \frac{2\varepsilon(1+S)}{(2S)^2}.
\]
\end{restatable}
It is further known that the logarithmic covering number of a bounded RKHS satisfies $\log \cN_T \asymp \gamma_T$, as discussed in~\cite[Example~5.12]{Dueling_Stackelburg_Krause,Wainwright_2019}. Building on this, we show in~\cref{sec:specified_ker}, for commonly used kernel functions, that --in line with~\cite{kayal2025bayesian}-- our results, and in particular their dependence on $\kappa$, improve upon the bounds of~\cite{Dueling-xu2024-kernel} and~\cite{Dueling_Stackelburg_Krause}.

\section{Experiments}
\phantomsection\label{sec:Experiments}
To corroborate our theoretical results, we conduct numerical experiments evaluating the performance of \textbf{GINOP} against several existing baselines across a range of reward function classes. Additional details on the experimental setup and results are deferred to~\cref{app:experiments}.

\begin{figure}[h]
    \centering
    \begin{subfigure}[b]{0.32\linewidth}
        \centering
        \includegraphics[width=\linewidth]{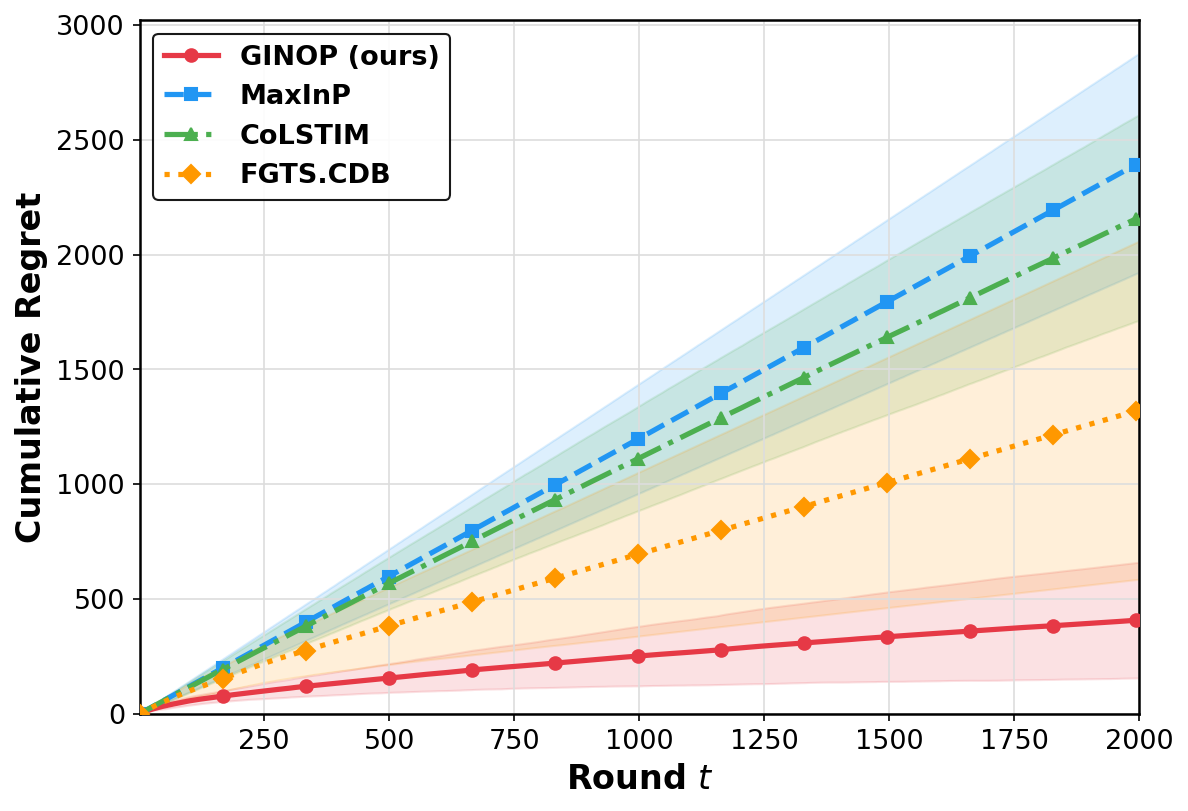}
        \caption{Linear utility.}
        \label{fig:first}
    \end{subfigure}
    \hfill
    \begin{subfigure}[b]{0.32\linewidth}
        \centering
        \includegraphics[width=\linewidth]{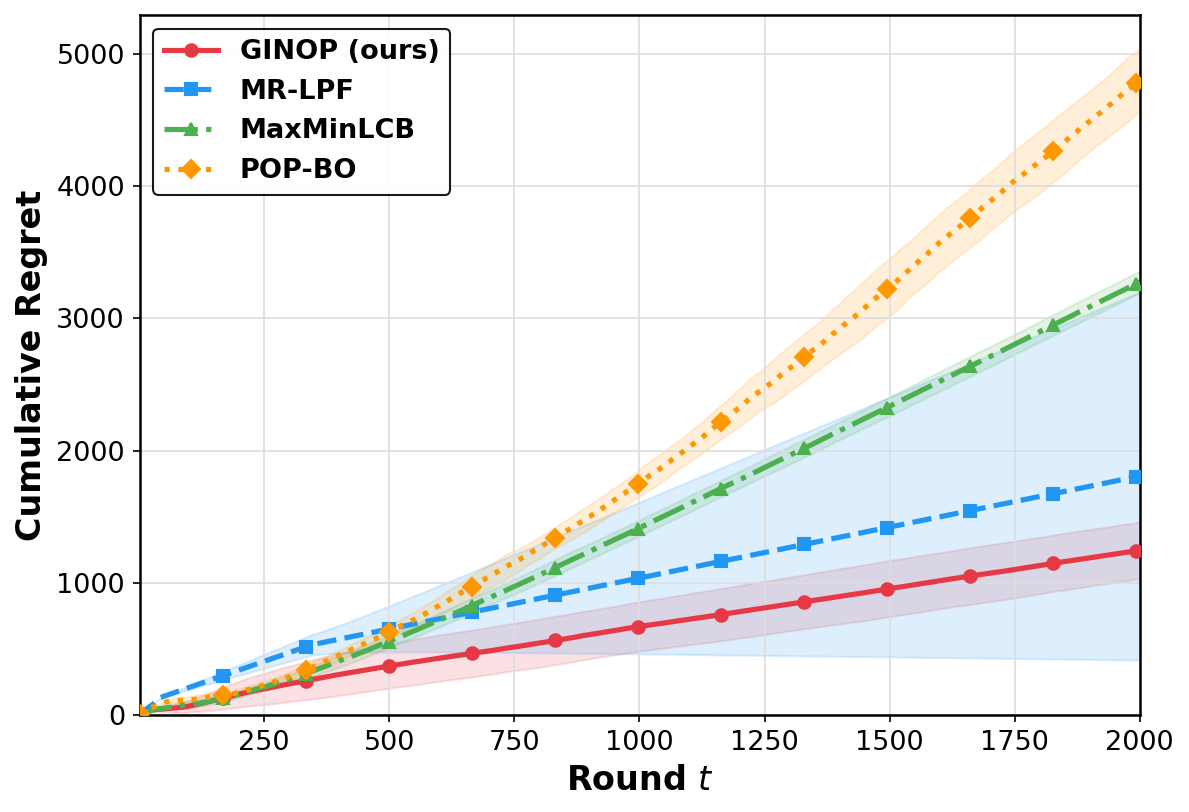}
        \caption{Kernelized utility.}
        \label{fig:second}
    \end{subfigure}
    \hfill
    \begin{subfigure}[b]{0.32\linewidth}
        \centering
        \includegraphics[width=\linewidth]{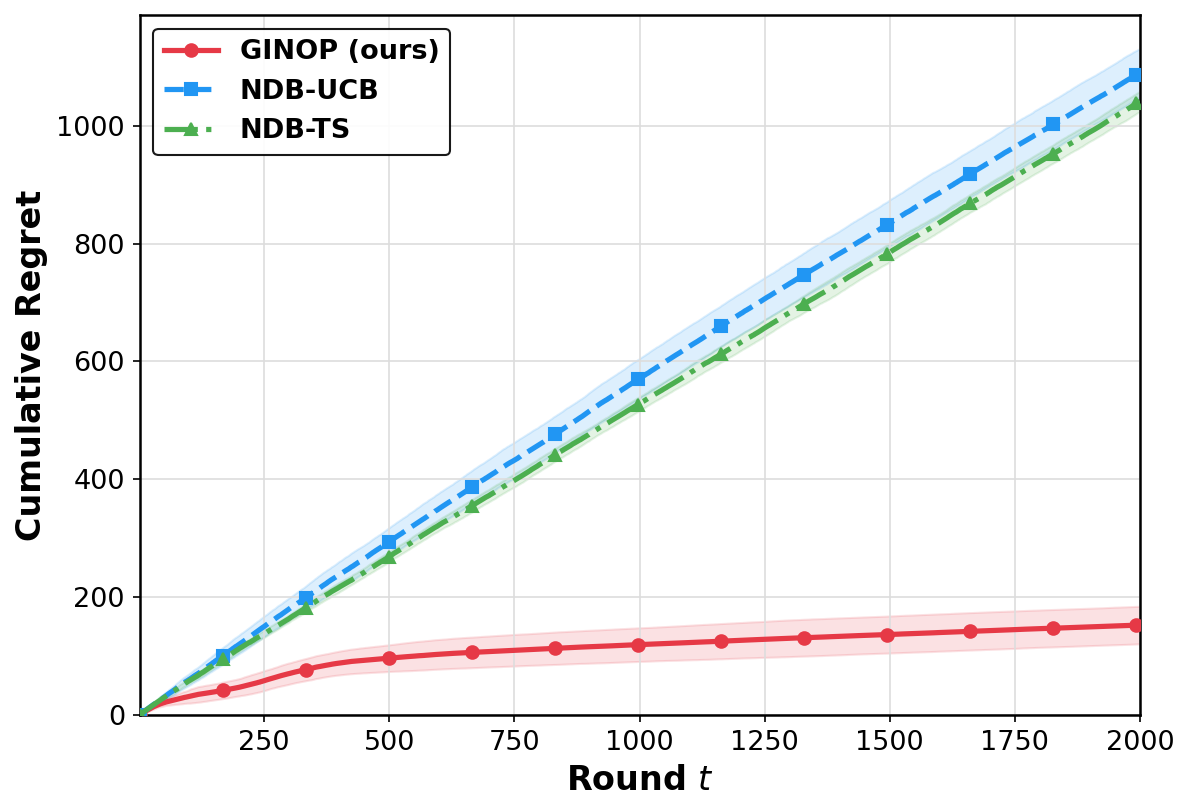}
        \caption{Against a NN.}
        \label{fig:third}
    \end{subfigure}
\caption{We benchmark \textbf{GINOP} against several baselines across different settings, reporting in each case the cumulative regret over a horizon of $T = 2000$ rounds, averaged over $20$ independent trials and considering an action set of $30$ arms. \textbf{(a)} In the linear setting, we take $d = 10$ and benchmark against \textbf{MaxInP}~\citep{Dueling-saha-linear}, \textbf{CoLSTIM}~\citep{pmlr-v162-bengs22a}, and \textbf{FGTS.CDB}~\citep{li2024feelgoodthompsonsamplingcontextual}. \textbf{(b)} In the kernelized setting, we adopt the Ackley function as the reward and employ the Matérn kernel with smoothness parameter $\nu = 2.5$; we benchmark against \textbf{POP-BO}~\citep{Dueling-xu2024-kernel}, \textbf{MaxMinLCB}~\citep{Dueling_Stackelburg_Krause}, and \textbf{MR-LPF}~\citep{kayal2025bayesian}. \textbf{(c)} In the neural setting, we consider a cosine reward function and benchmark against the neural-network-based approaches \textbf{NDB-UCB} and \textbf{NDB-TS}~\citep{verma2025neural}.}
    \label{fig:all}
\end{figure}

We observe that \textbf{GINOP} outperforms existing baselines, in agreement with our theoretical bounds, highlighting its favorable dependence on the curvature of the link function. Indeed,  all existing approaches (but \textbf{MR-LPF}) suffer from a $\kappa$-dependence both in the algorithmic design and in the theoretical performance, which inflates their regret and induces a $\sqrt{T}$-like shape that becomes apparent only at considerably larger horizons. This stresses the importance of the fine treatment of non-linearity in preference-based bandits. The exception is \textbf{MR-LPF} in the kernelized setting, which -- thanks to its multi-phase approach -- is able to eliminate the $\kappa$-dependence and consequently exhibits a regret profile comparable to that of \textbf{GINOP}. 

\section*{Conclusion and Future Work}

By leveraging a novel eluder dimension tailored to the preference feedback setting together 
with an informed-optimistic exploration scheme, we show that the resulting algorithm 
(\textbf{GINOP}) enjoys tight regret guarantees. These results hold for general reward 
function classes, thereby accommodating a wide range of problem instances. In particular, 
this highlights that the non-linearity of the logistic model -- inherited from the 
Bradley-Terry assumption -- does not adversely impact the regret in the preference bandit 
setting. Consequently, our work demonstrates that learning from preference feedback can be as statistically efficient as learning from direct reward observations.

While the primary focus of this paper is theoretical, on the computational front \textbf{GINOP} matches the computational complexity of existing work in specific settings (linear and kernelized utilities). Computational efficiency for general function classes, however, remains a common problem in the literature~\citep{sekhari2023contextual}; the design of suitable oracle-based implementations lies beyond the scope of this paper and is left for future work. In addition, promising future work (following  ~\cite{pmlr-v162-bengs22a}) lies in extending the analysis to observation models beyond the Bradley--Terry model.


\bibliographystyle{apalike}
\bibliography{Library}

\newpage
\appendix

\section*{Appendices}
\startcontents[appendices] 
\printcontents[appendices]{}{1}{%
  \section*{Appendix Contents}
  \mbox{}\hrulefill\par
}%

\newpage

\section{A Closer Look at the Localization Technique of~\texorpdfstring{\cite{Local2025eluder}}{Bakhtiari et al. (2025)}}
\phantomsection\label{sec:Closer_look_Loc_Elud}

\cite{Local2025eluder} build on the global $\ell_1$-eluder dimension of~\cite{liu-eluder-proof}, defined as follows.
\begin{definition}
Let $\cZ$ be a set, $\Psi$ a class of real-valued functions on $\cZ$, and 
$z = (z_1, z_2, \dots, z_n)$ a sequence of length $n$ in $\cZ$. We define the following.
\begin{enumerate}
    \item An element $z_t \in \cZ$ is said to be $\varepsilon$-independent of $z_{1:t-1}$ with respect 
    to $\Psi$ if there exists $\psi \in \Psi$ such that
    \[
    \sum_{i=1}^{t-1} |\psi(z_i)| \leq \varepsilon
    \quad \text{and} \quad
    |\psi(z_t)| > \varepsilon.
    \]
    \item The sequence $z$ is called an $\varepsilon$-eluder sequence with respect to $\Psi$ if, 
    for every $t \leq n$, $z_t$ is $\varepsilon$-independent of $z_{1:t-1}$ with respect 
    to $\Psi$.
    \item The $\varepsilon$-eluder dimension $\dim_{\mathrm{elud}}(\varepsilon;\, \Psi)$ 
    of $\Psi$ is the length of the longest $\varepsilon'$-eluder sequence with respect to 
    $\Psi$, maximized over all $\varepsilon' \geq \varepsilon$.
\end{enumerate}
\end{definition}

Building on their lower bound~\cite[Theorem~2]{Local2025eluder}, the authors establish that the global eluder dimension $\dim_{\mathrm{elud}}(1/T;\, \cF)$, taken over the entire class $\cF$, necessarily incurs an undesirable dependence on $\kappa$. This observation motivates a form of \emph{localization}: rather than considering the eluder dimension of the full class, they propose restricting it to a suitable subset $\cF' \subseteq \cF$ and deriving a regret upper bound in terms of $\dim_{\mathrm{elud}}(1/T;\, \cF')$ instead of $\dim_{\mathrm{elud}}(1/T;\, \cF)$. This refinement, at least ostensibly, removes the $\kappa$-dependence from the leading-order term. However, this comes at the cost of an opaque second-order term of the form 
\begin{align}
\operatorname{card}\{t \in [T] : f_t \notin \cF'\},\label{eq:second_term}
\end{align}
which is generally intractable: the subclass $\cF'$ is defined only implicitly and cannot be evaluated for general function classes.

To better illustrate the localization technique, the authors instantiate it in the generalized linear model setting, where they establish that the appropriate local family $\cF'$ permitting the removal of the $\kappa$-dependence from the leading term is given by~\cite[Section~4.2]{Local2025eluder}:
\[
\cF' \equiv \Theta' = \left\{ \theta \in \Theta : \forall x \in \cX,\; \abs{\langle x,\, \theta - \theta^\star \rangle} \leq 1/M \right\},
\]
where $M = 1/4$ is the self-concordance constant. Hence, the challenge lies in bounding the additional term:
\begin{align}
  \operatorname{card}\{t \in [T] : f_t \notin \cF'\} = \operatorname{card}\{t \in [T] : \exists x \in \cX,\; \abs{\langle x,\, \theta_t - \theta^\star \rangle} > 1/M\}  \label{eq:card}
\end{align}

Unfortunately, instead of upper bounding the term in~\eqref{eq:card} which requires sufficient accuracy \textit{uniformly over arms}, \cite[Proposition~5]{Local2025eluder} inaccurately bound the quantity: $$\operatorname{card}\{t \in [T] : \abs{\langle x_t,\, \theta_t - \theta^\star \rangle} > 1/M\},$$ which differs substantially from the correct second-order term~\eqref{eq:card} in that it requires sufficient accuracy only for actions played over the trajectory. While the latter \textit{on-policy} term is indeed logarithmic, it does not imply a bound on~\eqref{eq:card} which can be linear in $T$ in all generality. 

While we align with~\cite{Local2025eluder} regarding the shortcomings of the standard global eluder dimension, we depart from their approach: in~\cref{def:defsigmaeluder}, we introduce a novel \emph{locally sensitive eluder dimension} that explicitly and directly encodes the sensitivity of the link function within the definition itself, rather than relying on an external localization technique.
\section{Preliminaries}
\phantomsection\label{sec:concentrations}

\subsection{Concentrations}
We recall~\cref{lem:Ft_includes_r} stated in the main:

\lemFtIncludesr*

For each $f \in \cF$, we define the excess loss function
\(
\varphi_f : [0,1] \times \cX \times \cX \to \mathbb{R}
\)
and its corresponding expected excess loss
\(
\bar{\varphi}_f : \cX \times \cX \to \mathbb{R}
\)
as
\begin{align*}
&\varphi_f(y,x,x')
=
\ell\!\left(y, \sigma\!\left(\Delta_f(x, x')\right)\right)
-
\ell\!\left(y, \sigma\!\left(\Delta_{f^\star}(x, x')\right)\right), \\
&\bar{\varphi}_f(x,x')
=
\mathbb{E}_y\!\left[
\ell\!\left(y, \sigma\!\left(\Delta_f(x, x')\right)\right)
\right]
-
\mathbb{E}_y\!\left[
\ell\!\left(y, \sigma\!\left(\Delta_{f^\star}(x, x')\right)\right)
\right].
\end{align*}

With this in hand, we introduce Proposition~\ref{prop:csabsaConcentration}, which is a direct adaptation of Proposition~22 in~\cite{Local2025eluder}.

\begin{restatable}[]{proposition}{csabsaConcentration}
\phantomsection\label{prop:csabsaConcentration}
With probability at least $1 -\delta$, the following holds: 
        \begin{align*}
        \forall f \in \cF,\ \forall t \in [T], \quad 
        \sum_{s=1}^t \bar{\varphi}_f(x_s,x'_s)
        &\leq
        2 \left( \sum_{s=1}^t \varphi_f(y_s,x_s,x'_s) + \beta_t(\delta) \right). 
        \end{align*}
\end{restatable}

\begin{remark}\label{remark:GE}
 We denote by the event $\cG\cE$ the intersection of the good events in Lemma~\ref{lem:Ft_includes_r}, and Proposition~\ref{prop:csabsaConcentration}, which occurs with probability at least $1-2\delta$. This event will be used extensively throughout the proofs in the appendices.   
\end{remark}

\subsection{Bounding the On-Policy Prediction Error }
\phantomsection\label{sec:sigma._omega_upperbound_proof}

\lemSigmaOmegaproof*

\begin{proof}[Lemma~\ref{lem:sigma._omega_upperbound}]
Throughout the proof we suppose we are on $\cG\cE$. \\

\textbf{Step 1.} First we establish :
\begin{align}
    \forall t \in [T], \, \omega_t(x_t,x'_t)^2&=\Bigl(\sup_{f,f' \in \cF_t} \Delta_f(x_t,x'_t) - \Delta_{f'}(x_t,x'_t) \Bigr)^2 \nonumber\\
    &\leq \sup_{f,f' \in \cF_t} \Bigl( \Delta_f(x_t,x'_t) - \Delta_{f'}(x_t,x'_t) \Bigr)^2 \nonumber\\
    &\leq \sup_{f,f' \in \cF_t} \Bigl( \Delta_f(x_t,x'_t) - \Delta_{f^\star}(x_t,x'_t) + \Delta_{f'}(x_t,x'_t) - \Delta_{f^\star}(x_t,x'_t)\Bigr)^2 \nonumber\\
    &\overset{(i)}{\leq} \sup_{f,f' \in \cF_t} 2\Bigl( \Delta_f(x_t,x'_t) - \Delta_{f^\star}(x_t,x'_t) \Bigr)^2 + 2\Bigl(\Delta_{f'}(x_t,x'_t) - \Delta_{f^\star}(x_t,x'_t)\Bigr)^2 \nonumber\\
    &\le
    4 \sup_{f \in \cF_t}
    \bigl( \Delta_f(x_t,x'_t) - \Delta_{f^\star}(x_t,x'_t) \bigr)^2 \nonumber\\
    &=
    4 \bigl( \Delta_{f_t}(x_t,x'_t) - \Delta_{f^\star}(x_t,x'_t) \bigr)^2 \label{eq:omega_f_t}
\end{align}
where in $(i)$ we use that $(a+b)^2 \leq 2 a^2 + 2b^2$. For the last equality $f_t \in \cF_t$ denotes a function attaining the supremum.

\textbf{Step 2.} 
We apply~\cref{prop:csabsaConcentration}, for $f_t \in \cF$:
    \begin{align*}
        \sum_{s=1}^{t-1} \bar{\varphi}_{f_t}(x_s,x'_s)
        &\leq
        2 \left( \sum_{s=1}^{t-1} \varphi_{f_t}(y_s,x_s,x'_s) + \beta_t \right) \\
        &\leq 2 \left( \sum_{s=1}^{t-1} \ell(y_s, \sigma(\Delta_{f_t}(x_s, x'_s))) - \sum_{s=1}^{t-1} \ell(y_s, \sigma(\Delta_{f^\star}(x_s, x'_s))) + \beta_t(\delta) \right)  &&(\text{By definition of }\varphi_{f_t})\\
        &\leq 2 \left( \underbrace{\sum_{s=1}^{t-1} \ell(y_s, \sigma(\Delta_{f_t}(x_s, x'_s))) - \sum_{s=1}^{t-1} \ell(y_s, \sigma(\Delta{\hat{f}_t}(x_s, x'_s)))}_{\leq \beta_t(\delta): \, \text{because} \, f_t \in \cF_t} + \beta_t(\delta) \right)  &&(\text{By definition of }\varphi_{\hat{f}_t} \,  \text{and} \, f^\star \in \cF_t )\\
        &\leq 4 \beta_t(\delta). 
    \end{align*}
    Hence:
    \begin{align}
    {
        \sum_{s= 1}^{t-1} \bar{\varphi}_{f_t}(x_s, x'_s) \leq 4\beta_t(\delta) . \label{eq:upperbound_process}}
    \end{align}

\textbf{Step 3.} Recalling that our decision space is $\cX \times \cX$, we define
\[
\begin{aligned}
\Psi:\cF\times\cX\times\cX &\to \mathbb{R}_+, 
&\qquad
\phi:\cF\times\cX\times\cX &\to \mathbb{R}_+, \\
(f,x,x') &\mapsto \bar{\varphi}_{f}(x,x'),
&\qquad
(f,x,x') &\mapsto \dot{\sigma}(\Delta_{f^\star}(x,x'))\bigl(\Delta_{f}(x,x') - \Delta_{f^\star}(x,x')\bigr)^2 .
\end{aligned}
\]

For this choice of $\psi$ and $\phi$, the eluder $d_{\psi,\phi}(\cF, \varepsilon)$ from \cref{def:defsigmaloceluder} exactly coincide with $\delud(\Delta\cF, \Delta_{f^\star}, \varepsilon)$ from \cref{def:defsigmaeluder}. Using this, we apply \cref{lem:lemsumprederrors} with this choice of $\psi$ and $\phi$.

Let $\{(f_1,x_1,x'_1),\ldots,(f_T,x_T,x'_T)\}$ denote the corresponding sequence in
\cref{lem:lemsumprederrors}. Setting $\beta = 4\beta_t$ and  $C=S$ yields that:
\begin{align*}
\sum_{t=1}^T \sqrt{\dot{\sigma}(\Delta_{f^\star}(x,x'))}\bigl(\Delta_{f}(x,x') - \Delta_{f^\star}(x,x')\bigr)
&\le
2 \sqrt{4T\beta_T(\delta)\, \delud(\Delta\cF, \Delta_{f^\star}, \varepsilon)}
+
\sqrt{S}\, \delud(\Delta\cF, \Delta_{f^\star}, \varepsilon), \\
\sum_{t=1}^T \dot{\sigma}(\Delta_{f^\star}(x,x'))\bigl(\Delta_{f}(x,x') - \Delta_{f^\star}(x,x')\bigr)^2
&\le
4\beta_T(\delta)\, \delud(\Delta\cF, \Delta_{f^\star}, \varepsilon)
\left(
2 + \log\!\left(
\frac{T S}{16\beta_T(\delta)^2 d^2_{\sigma}(\Delta\cF, \Delta_{f^\star}, \varepsilon)}
\right)
\right) \nonumber \\
&\quad
+ \delud(\Delta\cF, \Delta_{f^\star}, \varepsilon)
\bigl(2+4\beta_T(\delta) \delud(\Delta\cF, \Delta_{f^\star}, \varepsilon)\bigr)(1+S) \\
\implies \sum_{t=1}^T \bigl(\Delta_{f}(x,x') - \Delta_{f^\star}(x,x')\bigr)^2
&\le
4\kappa\beta_T(\delta)\, \delud(\Delta\cF, \Delta_{f^\star}, \varepsilon)
\left(
2 + \log\!\left(
\frac{T S}{16\beta_T(\delta)^2 d^2_{\sigma}(\Delta\cF, \Delta_{f^\star}, \varepsilon)}
\right)
\right) \nonumber \\
&\quad
+ \kappa s \delud(\Delta\cF, \Delta_{f^\star}, \varepsilon)
\bigl(2+4\beta_T(\delta) \delud(\Delta\cF, \Delta_{f^\star}, \varepsilon)\bigr)(1+S) 
\end{align*}
Using~\eqref{eq:omega_f_t} implies the desired result:

\begin{align*}
\sum_{t=1}^T \sqrt{\dot{\sigma}(\Delta_{f^\star}(x_t,x'_t))} \omega_t(x_t,x'_t)
&\le
8 \sqrt{ T\beta_T(\delta)\, \delud(\Delta\cF, \Delta_{f^\star}, \varepsilon)}
+
2\sqrt{ S}\, \delud(\Delta\cF, \Delta_{f^\star}, \varepsilon)\\
&=\widetilde{\cO} \left( \sqrt{ T\beta_T(\delta)\, \delud(\Delta\cF, \Delta_{f^\star}, \varepsilon)} \right), \\
\sum_{t=1}^T \omega_t(x_t,x'_t)^2
&\le
16\kappa\beta_T(\delta)\, \delud(\Delta\cF, \Delta_{f^\star}, \varepsilon)
\left(
2 + \log\!\left(
\frac{T S}{16\beta_T(\delta)^2 d^2_{\sigma}(\Delta\cF, \Delta_{f^\star}, \varepsilon)}
\right)
\right) \nonumber \\
&\quad
+ 4\kappa \delud(\Delta\cF, \Delta_{f^\star}, \varepsilon)
\bigl(2+4\beta_T(\delta) \delud(\Delta\cF, \Delta_{f^\star}, \varepsilon)\bigr)(1+S)\\
&=\kappa \; \Gamma_T.
\end{align*}
where
\[
\Gamma_T
=
\widetilde{\cO}\!\Bigl(
\log( \cN_T(\Phi(\Delta \cF)) / \delta)\,
\delud(\Delta\cF, \Delta_{f^\star}, \varepsilon)^2
\Bigr).
\]



\end{proof}

\section{GINOP Performance}\label{sec:app-perf}

\subsection{Regret Decomposition}
\phantomsection\label{sec:round_wise_regret}

\LemRoundwiseError*
\begin{proof}[\cref{lem:roundwise-error}]
    Let $x_t$ and $x'_t$ be the arms selected by the algorithm at round $t$.  
Then the instantaneous regret can be decomposed as follows:
\begin{align}
    2 \, \rho(t) &= \Delta_{f^\star}(x^\star,x_t) + \Delta_{f^\star}(x^\star,x'_t) \nonumber\\
    &=  \Delta_{f^\star}(x^\star,x_t) - \Delta_{\hat{f}_t}(x^\star,x_t) + \Delta_{\hat{f}_t}(x^\star,x_t) + \Delta_{f^\star}(x^\star,x'_t) - \Delta_{\hat{f}_t}(x^\star,x'_t) + \Delta_{\hat{f}_t}(x^\star,x'_t) \nonumber
\end{align}

By the \cref{lem:Ft_includes_r} and the definition of $\omega_t$, then on $\cG\cE$:
\begin{align*}
    \forall t \in [T], \quad \Delta_{f^\star}(x^\star,x_t) - \Delta_{\hat{f}_t}(x^\star,x_t) & \leq \omega_t(x^\star,x_t)\\
    \Delta_{f^\star}(x^\star,x_t) - \Delta_{\hat{f}_t}(x^\star,x'_t) & \leq \omega_t(x^\star,x'_t)
\end{align*}
Hence:
\begin{align}
     2 \, \rho(t) &\leq \omega_t(x^\star,x_t) + \Delta_{\hat{f}_t}(x^\star,x_t) + \omega_t(x^\star,x'_t) + \Delta_{\hat{f}_t}(x^\star,x'_t). \label{eq:rho_upperboun1}
\end{align}
Given the optimization step (Step~\ref{step:Optimal_arms}), we obtain:
\begin{align}
    \hat{f}_t(x^\star) + \hat{f}_t(x'_t) + \omega_t(x^\star,x'_t) &\leq \hat{f}_t(x_t) + \hat{f}_t(x'_t) + \omega_t(x_t,x'_t) \nonumber\\
    \implies  \hat{f}_t(x^\star) -\hat{f}_t(x_t)  &\leq - \omega_t(x^\star,x'_t) + \omega_t(x_t,x'_t) \nonumber \\
    \implies  \Delta_{\hat{f}_t}(x^\star,x_t)  &\leq - \omega_t(x^\star,x'_t) + \omega_t(x_t,x'_t) \label{eq:ineq1}
\end{align}
Similarly, we obtain:
\begin{align}
    \Delta_{\hat{f}_t}(x^\star,x'_t)  &\leq - \omega_t(x^\star,x_t) + \omega_t(x_t,x'_t) \label{eq:ineq2}
\end{align}
Combining~\eqref{eq:rho_upperboun1}, \eqref{eq:ineq1} and \eqref{eq:ineq2} gives $\rho(t) \leq \omega(x_t,x'_t)$.
\end{proof}

\subsection{Regret Bound}
\phantomsection\label{sec:Alg_Performance_Proof}

\ThmAlgproof*
\begin{proof}[\cref{thm:Alg_perf}] \hspace{1cm} \hspace{1cm}

\paragraph{Step.1} We work under the good event $\cG\cE$. We define $\dot{\sigma}^\star = \dot{\sigma}(\Delta_{f^\star}(x^\star,x^\star)) =\sfrac{1}{4}$.

\begin{align*}
    \cR_T = \sum_{t=1}^T \rho(t) 
    &\leq  \sum_{t=1}^T \omega_t(x_t,x'_t)=  \frac{1}{\sqrt{\dot{\sigma}^\star}} \sum_{t =1}^T \sqrt{\dot{\sigma}^\star} \; \omega_t(x_t,x'_t). 
\end{align*}
Using a first order Taylor expansion on $\dot{\sigma}$:
\begin{align*}
    \dot{\sigma}^\star &\leq \dot{\sigma}(\Delta_{f^\star}(x_t,x'_t)) + \frac{1}{4} \abs{\Delta_{f^\star}(x_t,x'_t) -\Delta_{f^\star}(x^\star,x^\star)} &&(|\ddot{\sigma}| \leq |\dot{\sigma}| \leq \sfrac{1}{4}) \\
    &\leq \dot{\sigma}(\Delta_{f^\star}(x_t,x'_t)) + \frac{\rho(t)}{2}.
\end{align*}
Hence:
\begin{align*}
    \cR_T &\leq  \frac{1}{\sqrt{\dot{\sigma}^\star}} \sum_{t =1}^T \sqrt{\dot{\sigma}^\star} \; \omega_t(x_t,x'_t) \\
    &\leq \frac{1}{\sqrt{\dot{\sigma}^\star}} \sum_{t =1}^T \sqrt{\dot{\sigma}(\Delta_{f^\star}(x_t,x'_t))} \; \omega_t(x_t,x'_t) + \frac{1}{\sqrt{2\dot{\sigma}^\star}} \sum_{t =1}^T \sqrt{\rho(t)} \; \omega_t(x_t,x'_t) \\
    &\leq \underbrace{\frac{1}{\sqrt{\dot{\sigma}^\star}} \sum_{t =1}^T \sqrt{\dot{\sigma}(\Delta_{f^\star}(x_t,x'_t))} \; \omega_t(x_t,x'_t)}_{a} + \underbrace{\frac{1}{\sqrt{2\dot{\sigma}^\star}} \sqrt{\sum_{t =1}^T  \omega_t(x_t,x'_t)^2} }_{b}\sqrt{\cR_T} && (\text{Cauchy Schwartz}) 
\end{align*}
Which leads to the following quadratic inequality on $\cR_T$:
\begin{align*}
     \cR_T &\leq a + b \sqrt{\cR_T} \\
    \implies\sqrt{\cR_t} & \leq \sqrt{a}+ b\\
    \implies \cR_t &\leq 2a + 2b^2.
\end{align*}
By Lemma~\ref{lem:sigma._omega_upperbound}:
\begin{align*}
    a &= \widetilde{\cO}\left( \frac{1}{\sqrt{\dot{\sigma}^\star}}\sqrt{ T\beta_T(\delta)\, \delud(\Delta\cF,f^\star,\varepsilon)} \right) \\
    b^2 &= \frac{\kappa}{2 \dot{\sigma}^\star} \Gamma_T.
\end{align*}
Hence on $\cG\cE$: $\cR_T \leq \widetilde{\cO}\Bigl(\frac{\kappa}{2 \dot{\sigma}^\star} \Gamma_T + \frac{1}{\sqrt{\dot{\sigma}^\star}}\sqrt{ T\beta_T(\delta)\, \delud(\Delta\cF, f^\star, \varepsilon)}\Bigr)$
\paragraph{Step 2.}
\begin{align*}
    R_T
     &= \lE \left[ \cR_T \mid \cG\cE \right]\P(\cG\cE) + \lE \left[ \cR_T \mid \overline{\cG\cE} \right]\P(\overline{\cG\cE}) \\
     &\leq  \lE \left[ \cR_T \mid \cG\cE \right] + \lE \left[ \sum_{t=1}^T \rho(t) \mid \overline{\cG\cE} \right]2\delta \\
     &\leq \widetilde{\cO}\Bigl(\frac{\kappa}{2 \dot{\sigma}^\star} \Gamma_T + \frac{1}{\sqrt{\dot{\sigma}^\star}}\sqrt{ T\beta_T(\delta)\, \delud(\Delta\cF, \Delta_{f^\star}, \varepsilon)}\Bigr) + 4ST\delta  \\
     &\leq \widetilde{\cO}\Bigl(\frac{\kappa}{2 \dot{\sigma}^\star} \Gamma_T + \frac{1}{\sqrt{\dot{\sigma}^\star}}\sqrt{ T\beta_T(\delta)\, \delud(\Delta\cF, \Delta_{f^\star}, \varepsilon)}\Bigr)    && (\text{Because: }\delta= \frac{1}{T^2}).
\end{align*}
Which concludes the proof of the regret bound.

\end{proof}

\section{On the Eluder Dimension}

We state all the results of this section for a more generic version of \cref{def:defsigmaeluder}.

\begin{definition}[$\varepsilon$-independence]
  \phantomsection\label{def:epsilon-independence}
  Let $\cQ$ be a class of functions $q: \cZ \to \mathbb{R}$. 
  Let $\psi$ and $\phi$ be two functions from $\cQ\times\cZ\to \mathbb{R}_+$.
  Given a sequence $(z_1,\ldots, z_n) \in \cZ^{n}$ of length $n>0$ and $\varepsilon > 0$,
  we say that $z \in \cZ$ is $(\varepsilon, \psi, \phi)$-independent from the sequence
  with respect to $\cQ$ if there exists $q\in\cQ$ such that
  \begin{align*}
    \sum_{j \leq n} \psi(q, z_j)\leq \varepsilon^2
    \quad \text{and} \quad
    \phi(q, z) > \varepsilon^2.
  \end{align*}
\end{definition}

\begin{restatable}[$(\varepsilon, \psi, \phi)$-Eluder]{definition}{defsigmaloceluder}
  \phantomsection\label{def:defsigmaloceluder}
  Let $\cQ$ be a class of functions $q: \cZ \to \mathbb{R}$.
  Let $\psi$ and $\phi$ be two functions from $\cQ\times\cZ\to \mathbb{R}_+$.
  The $(\varepsilon, \psi, \phi)$-eluder dimension $d_{\psi, \phi}(\cQ, \varepsilon)$ of $\cQ$ is the length $d$ of the longest sequence $(z_1,\ldots, z_d) \in \cZ^d$ for which 
    there exists $\varepsilon' > \varepsilon$ such that for any $i\leq d$, $z_i$ is $(\varepsilon', \psi, \phi)$-independent from $(z_1, \ldots, z_{i-1})$ with respect to $\cQ$.
\end{restatable}

\begin{restatable}[Generalization of Proposition 3 of \cite{Eluder_russo}]{lemma}{lemsigmaloceluder}
    \phantomsection\label{lem:lemsigmaloceluder}
    Let $\cZ$ be a set and $\cQ$ be a class of functions $q: \cZ \to [-S, S]$.
    Let $\psi$ and $\phi$ be two functions from $\cQ\times\cZ\to \mathbb{R}_+$. 
    Additionally, let's assume that $\phi$ is uniformly bounded above by some $C>0$.
    Suppose sequences $(z_1, q_1), \ldots, (z_T, q_T) \in \cZ\times \cQ$ and $\beta>0$ such that,
    for all $t\leq T$,
    \begin{align}
        \sum_{j \leq t} \psi(q_t, z_j)\leq \beta\,,
    \end{align}
    Then, for any $\varepsilon > 0$, 
    \begin{align}
        \sum_{j \leq T} \indicator{\phi(q_j, z_j) > \varepsilon^2} \leq \left(\frac{\beta}{\varepsilon^2}+1\right)d_{\psi, \phi}(\cQ, \varepsilon)\,.
    \end{align}
\end{restatable}

\begin{proof}[\cref{lem:lemsigmaloceluder}]
    The proof is essentially the one of Proposition 3 of \cite{Eluder_russo} that adapts almost directly. It is re-derived here for the sake of completeness. For readability, we will drop the dependence on $\psi$ and $\phi$ in the notation of (in)dependence.

    \textbf{Claim 1.} Let $\varepsilon>0$. Then if for some $t\leq T$, we have $\phi(q_t, z_t) > \varepsilon^2$, then
    $z_t$ is $\varepsilon$-dependent with respect to $\cQ$ on at most $\beta/\varepsilon^2$ disjoint subsequences in $(z_1, \ldots, z_{t-1})$.

    \textit{Proof of Claim 1.} For $t$ such that $\phi(q_t, z_t) > \varepsilon^2$, 
    let $(z_1, \ldots, z_l)$ be a subsequence of $(z_1, \ldots, z_{t-1})$
    such that $z_t$ is $\varepsilon$-dependent on $(z_1, \ldots, z_l)$ with respect to $\cQ$.
    Then, necessarily, $\sum_{i\leq l} \psi(q_t, z_i) > \varepsilon^2$ 
    (otherwise, $z_t$ would not be $\varepsilon$-dependent). 
    Then, if $z_t$ is $\varepsilon$-dependent on $L$ disjoint subsequences of $(z_1, \ldots, z_{t-1})$, 
    we have $\sum_{i<t} \psi(q_t, z_i) > \varepsilon^2 L$.
    Since $\sum_{i<t} \psi(q_t, z_i) \leq \beta$, we have $L < \beta/\varepsilon^2$.

    \textbf{Claim 2.} Let $\cQ$ be a class of functions $q: \cZ \to \lR$. For any $\varepsilon>0$ and
    sequence $(z_1, \ldots, z_\tau) \in \cZ^\tau$, there exists $j\leq\tau$ such that
    $z_j$ is $\varepsilon$-dependent with respect to $\cQ$ on at least $L \geq \tau/d_{\psi, \phi}(\cQ, \varepsilon)-1$
     disjoint subsequences in $(z_1, \ldots, z_{\tau-1})$.

    \textit{Proof of Claim 2.} Let choose an integer $L$ such that $L d_{\psi, \phi}(\cQ, \varepsilon) + 1 \leq \tau \leq (L+1) d_{\psi, \phi}(\cQ, \varepsilon)$.
    Let $B_i = (z_i)$ for all $i = 1,\ldots,L$. If $z_{L+1}$ is $\varepsilon$-dependent on each subsequence $B_1, \ldots, B_L$, the claim is established. Otherwise, choose a subsequence $B_i$ such that $z_{L+1}$ is $\varepsilon$-independent of $B_i$ and append it $B_i := B_i \cup (z_{L+1})$. Repeat this process for elements with indices $j>L+1$ until $z_j$ is $\varepsilon$-dependent on each subsequence or $j=\tau$. If $j<\tau$, the proof is finished. 
    Let's focus on the case $j=\tau$.
    Since for any $i\leq L$, each element of $B_i$ is $\varepsilon$-independent of its predecessors, $|B_i| \leq d_{\psi, \phi}(\cQ, \varepsilon)$.
    Moreover, $\sum_{i\leq L}|B_i| = \tau-1 \geq L d_{\psi, \phi}(\cQ, \varepsilon)$. 
    Thus, for any $i\leq L$, $|B_i| = d_{\psi, \phi}(\cQ, \varepsilon)$.
    Then $z_\tau$ must be $\varepsilon$-dependent on each subsequence by definition of $d_{\psi, \phi}(\cQ, \varepsilon)$.

    Putting everything together, let $(z_{t_1}, \ldots, z_{t_\tau})$ be the sub-sequence of $(z_1, \ldots, z_T)$ consisting in indexes $t$ for which 
    $\phi(q_t, z_t) > \varepsilon^2$. By applying Claim 1 and 2 to this sub-sequence, we obtain that 
    \begin{align*}
        \left\lfloor\frac{\tau}{d_{\psi, \phi}(\cQ, \varepsilon)}\right\rfloor-1 < \frac{\beta}{\varepsilon^2}
        \Leftrightarrow 
        \tau < \left(\frac{\beta}{\varepsilon^2}+1\right)d_{\psi, \phi}(\cQ, \varepsilon)
    \end{align*}

\end{proof}

We now restate (a simplified version of) Lemma 21 from \cite{croissant24} before applying it.

\begin{restatable}[Lemma 21 in \cite{croissant24}]{lemma}{lemsums}
    \phantomsection\label{lem:lemsums}
    Let $(z_t)_{t\in\lN} \in [0,Z]^\lN$, $(b_t)_{t\in\lN} \in (\lR\to\lR_+)^\lN$ a sequence of non-increasing functions and $a \in (\lR\to\lR_+)$ non-increasing such that for any $T > 0$ and $\varepsilon>0$
    \begin{align}
        \sum_{t=1}^T \indicator{z_t>\varepsilon} \leq \frac{b_T(\varepsilon)}{\varepsilon^2} + a(\varepsilon)
        \quad\text{and}\quad
        \varepsilon \leq \frac{b_T(\varepsilon)\wedge\sqrt{b_T(\varepsilon)}}{\sqrt{T}}
        \label{eq:lemma21condition}
    \end{align}
    Then, the following two inequalities hold, 
    \begin{align}
        & \sum_{t=1}^T z_t \leq 2 \sqrt{Tb_T(\varepsilon)} + Z a(\varepsilon)\\
        & \sum_{t=1}^T z_t^2 \leq b_T(\varepsilon) \left(2+\log\left(\frac{TZ^2}{b_T(\varepsilon)^2}\right)\right) + a(\varepsilon)(2+b_T(\varepsilon))(1+Z^2)
    \end{align}
\end{restatable}

\begin{restatable}[]{lemma}{lemsumprederrors}
    \phantomsection\label{lem:lemsumprederrors}
    Let $\cZ$ be a set and $\cQ$ be a class of functions $q: \cZ \to [-S, S]$.
    Let $\psi$ and $\phi$ be two functions from $\cQ\times\cZ\to \mathbb{R}_+$. 
    Additionally, let's assume that $\phi$ is uniformly bounded above by some $C>0$.
    Suppose sequences $(z_1, q_1), \ldots, (z_T, q_T) \in \cZ\times \cQ$ and $\beta>0$ such that,
    for all $t\leq T$,
    \begin{align}
        \sum_{j \leq t} \psi(q_t, z_j)\leq \beta\,,
    \end{align}
    Then, for any $0<\varepsilon \leq \frac{\beta}{\sqrt{T}}$, 
    \begin{align}
        & \sum_{t=1}^T \sqrt{\phi(q_t, z_t)} \leq 2 \sqrt{T\beta d_{\psi, \phi}(\cQ, \varepsilon)} + \sqrt{C} d_{\psi, \phi}(\cQ, \varepsilon)\\
        & \sum_{t=1}^T \phi(q_t, z_t) \leq \beta d_{\psi, \phi}(\cQ, \varepsilon) \left(2+\log\left(\frac{TC}{\beta^2d^2_{\psi, \phi}(\cQ, \varepsilon)}\right)\right) + d_{\psi, \phi}(\cQ, \varepsilon)(2+\beta d_{\psi, \phi}(\cQ, \varepsilon))(1+C)\,.
    \end{align}
\end{restatable}

\begin{proof}[\cref{lem:lemsumprederrors}]
Apply \cref{lem:lemsums} with $z_t = \sqrt{\phi(q_t, z_t)}$, $b_t(\varepsilon) = \beta d_{\psi, \phi}(\cQ, \varepsilon)$ and $a(\varepsilon) = d_{\psi, \phi}(\cQ, \varepsilon)$, noticing that
\begin{enumerate}
    \item $d_{\psi, \phi}(\cQ, \varepsilon)$ is non-increasing as a function of $\varepsilon$,
    \item as $d_{\psi, \phi}(\cQ, \cdot) \geq 1$, the condition $\varepsilon \leq \frac{\beta}{\sqrt{T}}$ is sufficient to satisfy the r.h.s. of \cref{eq:lemma21condition}.
\end{enumerate}

\end{proof}

\section{Instantiation to Familiar Utility Classes}

\subsection{Linear Utility}
\phantomsection\label{sec:Eluder_for_linear}
In this setting, $\cX \subset \mathbb{B}_2^d(1)$ and, for any $x \in \cX$, the reward function is given by
\(
f^\star(x) = \langle \theta^\star, x \rangle,
\)
where $\theta^\star \in \mathbb{R}^d$ is an unknown parameter satisfying $\|\theta^\star\| \leq S$. Accordingly, the class of linear functions $\cF_\Theta$ can be identified with the parameter set $\Theta \subset \mathbb{B}_2^d(S)$. Under the Bradley--Terry model, the preference feedback satisfies
\(
y_t \sim \mathrm{Bernoulli}\!\left(\sigma\!\left(\langle \theta^\star, x_t - x'_t \rangle\right)\right).
\)

\paragraph{Algorithmic Implementation.} Let $\hat{\theta}_t$ denote the MLE given by~\cref{eq:f_hat}. The corresponding confidence set becomes:
\(
\Theta_t
=
\left\{
\theta \in \Theta :
\cL(\theta; \cH_{t-1})
\le
\cL(\hat{\theta}_t; \cH_{t-1})
+ \beta_t(\delta)
\right\}.
\)
This set can be further enclosed within an ellipsoid of the following form (see  \cite{lee2024a} \hspace{1em} for \hspace{1em }instance):
\(
\Theta_t
\subset
\left\{
\theta \in \Theta :
\|\theta - \hat{\theta}_t\|^2_{\nabla^2 \cL(\hat{\theta}_t; \cH_{t-1})}
\le 2(1+2S)\beta_t(\delta)
\right\}.
\) Moreover, the uncertainty on the performance gap in~\cref{eq:omega_def} admits a closed-form characterization: 
\[\forall (x,x') \in \cX^2, \;
\omega_t(x,x')^2
\propto
\beta_t \, \|x - x'\|^2_{(\nabla^2 \cL(\hat{\theta}_t; \cH_{t-1}))^{-1}} \overset{\text{\cite{Faury20a}}}{\propto}
\beta_t \, \|x - x'\|^2_{H_t^{-1}}
\]
where $H_t = \sum_{s=1}^{t-1} \dot{\sigma}((x_s-x'_s)^\top \hat{\theta}_t)(x_s - x'_s)(x_s - x'_s)^\top$. Therefore, our approach differs from prior work in two ways. First, regarding the decision rule, existing methods typically proceed in two stages: either by constructing a set of plausible optimal arms and then selecting a pair that maximizes an exploration bonus (see, e.g.,~\cite{Dueling-saha-linear,di2024varianceaware}), or by first selecting a reference arm using an upper-confidence bound and then choosing a second arm that maximizes a relative optimistic criterion (see, e.g.,~\cite{pmlr-v162-bengs22a}). In contrast, our algorithm selects the pair of arms jointly in a single step. Second, the uncertainty bonus is fundamentally different. Prior approaches~\citep{Dueling-saha-linear,pmlr-v162-bengs22a,li2024feelgoodthompsonsamplingcontextual} rely on quantities of the form $\|x - x'\|_{V_t^{-1}}$, where $V_t = \sum_{i=1}^{t-1} (x_i - x'_i)(x_i - x'_i)^\top$. In contrast, \textbf{GINOP} leverages a locally adaptive design matrix $H_t$ that incorporates the curvature of the sigmoid through $\dot{\sigma}$. This refinement more accurately captures the nature of observations and avoids the undesirable dependence on $\kappa$ in the leading term of the regret, as highlighted in the subsequent analysis.

To correctly assess \textbf{GINOP} regret in the case of linear utility, we need first to evaluate the key complexity quantities appearing in the theorem, namely $\Log{\cN_T(\Phi(\Delta \cF))}$ and $\delud(\Delta\cF, \Delta_{f^\star}, \varepsilon)$, when $\cF = \cF_{lin}$.

\propsigmaeluderlinear*

\begin{proof}[\cref{prp:propsigmaeluderlinear}]\hspace{1cm}
    
\paragraph{We start with $\delud(\Delta\cF_{lin}, \Delta_{f^\star}, \varepsilon)$.} 

Let $((x_1,x'_1),\dots,(x_n,x'_n)) \in \mathcal{X}^{2n}$ and $(\theta_0,\dots,\theta_n) \subseteq \mathbb{B}_2^d(S)$ be the longest $\varepsilon$-independent sequence. For simplicity we denote by $z_i=x_i-x'_i$. Then there exists $\varepsilon' > \varepsilon$ such that for all $t \le n$,
\[
 \sum_{i=1}^{t-1} \bar{\varphi}(\sigma(z_i^\top\theta^\star),\sigma(z_i^\top\theta_{t})) \le \varepsilon',
 \qquad
 \dot{\sigma}(z_t^\top \theta^\star) (z_t^\top \big(\theta^\star-\theta_t)\big)^2 \ge \varepsilon'.
\]

A second-order Taylor expansion yields, for all $i<t \le n$,
\begin{align*}
\bar{\varphi}(\sigma(z_i^\top\theta^\star),\sigma(z_i^\top\theta_{t}))
 = \bar{\varphi}(\sigma(z_i^\top\theta^\star),\sigma(z_i^\top\theta^\star) )
 &+ \left[ \sigma(z_i^\top \theta^\star) - \sigma(z_i^\top \theta^\star)\right]z_i^\top (\theta_t-\theta^\star) \\
  &+ \left[z_i^\top (\theta_t - \theta^\star) \right]^2 \underbrace{\int_{0}^1(1-\nu) \dot{\sigma}(z_i^\top \theta^\star + \nu z_i^\top (\theta_t-\theta^\star))d\nu}_{(A)}.
\end{align*}

Using Lemma~8 of~\cite{pmlr-v130-abeille21a}., we obtain
\[
 (A) \geq \frac{\dot{\sigma}(z_i^\top \theta^\star)}{2 + \abs{z_i^T(\theta_t -\theta^\star)}} \ge \frac{\dot{\sigma}(z_i^\top \theta^\star)}{2(1+S)}.
\]
Hence, for all $t \le n$,
\begin{align}
 \sum_{i=1}^{t-1} \dot{\sigma}(z_i^\top \theta^\star) (z_i^\top (\theta^\star-\theta_t))^2 \le 2(1+S)\varepsilon' \label{eq:eq1_exlin},
\end{align}
while
\begin{align}
 \dot{\sigma}(z_t^\top \theta^\star) (z_t^\top (\theta^\star-\theta_t))^2 \ge \varepsilon'. \label{eq:eq2_exlin}
\end{align}

Define
\[
 \bar{H}_t(\theta^\star) = \sum_{i=1}^{t-1} \dot{\sigma}(z_i^\top \theta^\star) z_i z_i^\top,
 \qquad H_t(\theta^\star) = \bar{H}_t(\theta^\star) + \lambda I,
\]
with $\lambda = \varepsilon'2(1+S)/(2S)^2$. From equation~\eqref{eq:eq1_exlin}:
\[
 \|\theta^\star-\theta_t\|_{H_t(\theta^\star)}^2 \leq  \|\theta^\star-\theta_t\|_{\bar{H}_t(\theta^\star)}^2 + \lambda (2S)^2 \leq 4 (1+S)\varepsilon'
\]

From equation~\eqref{eq:eq2_exlin}:
\begin{align*}
 \varepsilon' \leq \dot{\sigma}(z_t^\top \theta_t) \norm{z_t}^2_{H_t(\theta^\star)^{-1}} \norm{\theta_t-\theta^\star}^2_{H_t(\theta^\star)} &\leq 4(1+S)\varepsilon' \dot{\sigma}(z_t^\top \theta^\star) \norm{z_t}^2_{H_t(\theta^\star)^{-1}}  \\
 \implies \dot{\sigma}(z_t^\top \theta^\star) \norm{z_t}^2_{H_t(\theta^\star)^{-1}}  &\geq \frac{1}{4(1+S)}.
\end{align*}
Using $ \abs{H_t(\theta^\star)}= \abs{H_{t-1}(\theta^\star)}\Big( 1 + \dot{\sigma}(z_t^\top \theta^\star)\norm{z_t}_{H_t(\theta^\star)^{-1}}^2\Big) $ we obtain:
\[ \abs{H_n(\theta^\star)} \geq \Big( 1 + \frac{1}{4(1+S)}\Big)^{n-1} \lambda^d\]
Further:
\[ \abs{H_n(\theta^\star)} \leq \Big( \frac{Tr(H_n(\theta^\star))}{d}\Big)^d \leq \Big( \lambda + \frac{n-1}{4d}\Big)^d\]
which implies :
\[ \Big(1 + \frac{n-1}{4d\lambda} \Big) \geq \Big( 1 + \frac{1}{4(1+S)}\Big)^{\frac{n-1}{d}}\]
The rest of the proof directly follows ~\cite{Eluder_russo}'s proof of Proposition 6 (step 3). We obtain:
\[ n  \leq 1 + d\left[10(1+S)\Log{1 + \frac{2(1+S)}{\varepsilon'}} + \ln{5(1+S)} \right]\]

And hence:
\[{
n =\cO \left(d\Log{1+ \sfrac{1}{\epsilon}} \right).}\]

\paragraph{Bounding $\Log{\cN_T(\Phi(\Delta \cF))}$.}

Let $\Phi(\Delta_{f_\theta}) \in \Phi(\Delta\cF_{lin})$. For all $y \in [0,1]$ and $x,x' \in \cX$, we have
\begin{align*}
\Phi(\Delta_{f_\theta})(y,x,x') 
&= \ell\bigl(y, \sigma(\Delta_{f_\theta}(x,x'))\bigr) - \ell\bigl(y, \sigma(\Delta_{f_{\theta^\star}}(x,x'))\bigr), \\
\nabla_\theta \Phi(\Delta_{f_\theta})(y,x,x')
&= \nabla_\theta \ell\bigl(y, \sigma(\Delta_{f_\theta}(x,x'))\bigr) \\
&= \bigl(\sigma(\langle \theta, x - x' \rangle) - y\bigr)(x - x').
\end{align*}
Since $\sigma(\cdot) \in [0,1]$ and $\|x - x'\| \leq 2$, it follows that
\[
\|\nabla_\theta \Phi(\Delta_{f_\theta})(y,x,x')\| \leq 2.
\]
Hence, the mapping $\theta \mapsto \Phi(\Delta_{f_\theta})$ is Lipschitz continuous. By Example~3 in~\cite{Eluder_russo}\phantomsection\label{remrk:CoveringForSmoothFamily}, we obtain
\[
\Log{\cN_T(\Phi(\Delta \cF))} = \cO\!\left(d \log T\right).
\]

\end{proof}

\subsection{Kernelized Utility}
\phantomsection\label{sec:Eluder_for_Ker}

Similar to \citep{Dueling_Stackelburg_Krause, Dueling-xu2024-kernel,kayal2025bayesian}, we assume that the utility function $f^\star$ belongs to a known Reproducing Kernel Hilbert Space (RKHS).  Let $k : \cX \times \cX \to \lR$ be a positive definite kernel, and let $\cH_k$ denote the associated RKHS. The space $\cH_k$ is equipped with inner product $\langle \cdot, \cdot \rangle_{\cH_k}$ and norm $\|\cdot\|_{\cH_k}$. We assume $\|f^\star\|_{\cH_k}\leq S$ and $k(x,x) \leq 1$ for all $x \in \cX$. By the reproducing property, for all $f \in \cH_k$ and $x \in \cX$, it holds that
\(
\langle f, k(\cdot, x) \rangle_{\cH_k} = f(x).
\) By Mercer’s theorem, under mild conditions, the kernel admits the representation
\(
k(x, x') = \sum_{m=1}^{\infty} \gamma_m \phi_m(x)\phi_m(x'),
\)
where $\gamma_m > 0$, and $\{\psi_m:=\sqrt{\gamma_m}\phi_m\}_{m\geq 1}$ forms an orthonormal basis of $\cH_k$. In particular, any $f \in \cH_k$ can be expressed as
\(
f(\cdot) = \sum_{m=1}^{\infty} \theta_m \psi_m(\cdot)= \theta^\top \Psi(\cdot), \quad \text{with} \quad \|f\|_{\cH_k}^2 = \sum_{m=1}^{\infty} \theta_m^2 \leq S.
\)
We refer to $\{\gamma_m\}$ and $\{\phi_m\}$ as the (Mercer) eigenvalues and eigenfunctions of $k$, respectively.

Define $z = (x, x') \in \cX \times \cX$ and $\Delta_f(z) = f(x) - f(x')$. Following \cite{Dueling_Stackelburg_Krause,kayal2025bayesian}, we introduce the \emph{dueling kernel}
\[
\mathds{k}(z_1, z_2) = k(x_1, x_2) + k(x_1', x_2') - k(x_1, x_2') - k(x_1', x_2),
\]
for $z_1 = (x_1, x_1')$ and $z_2 = (x_2, x_2')$. This construction satisfies $\|\Delta_f\|_{\cH_{\mathds{k}}} = \|f\|_{\cH_k}$ (see \cite[Proposition 4]{Dueling_Stackelburg_Krause}). Moreover, both $f$ and $\Delta_f$ share the same Mercer coefficient vector $\theta \in \ell_2(\lN)$; that is, for all $x, x' \in \cX$,
\(
f(x) = \theta^\top \Psi(x)
\quad \text{and} \quad
\Delta_f(x,x') = \theta^\top \bigl(\Psi(x) - \Psi(x')\bigr).
\)



\paragraph{Algorithmic Implementation.}
Invoking Mercer's theorem, and following a similar approach to~\citep{Dueling_Stackelburg_Krause,kayal2025bayesian}, we consider the regularized negative log-likelihood loss
\[
\cL_{\mathds{k}}(\theta; \cH_{t-1})
:= \sum_{s=1}^{t-1} 
\ell\Bigl(y_s, \sigma\!\bigl(\theta^\top(\Psi(x_s)-\Psi(x'_s))\bigr)\Bigr)
+ \frac{\lambda}{2}\|\theta\|^2 .
\]
The MLE estimator is defined as $\hat{f}_t=\hat{\theta}_t^\top \Psi$, where
\(
\hat{\theta}_t
= \argmin_{\theta \in \ell_2(\lN)}
\cL_{\mathds{k}}(\theta; \cH_{t-1}) .
\) A tractable counterpart of this infinite-dimensional optimization problem is obtained by leveraging the Representer Theorem. Specifically,
\(
\alpha_t
= \argmin_{\alpha \in \lR^{t-1}}
\sum_{s=1}^{t-1}
\ell\Bigl(y_s, \sigma\!\bigl(\alpha^\top \mathds{k}_{t-1}(x_s,x'_s)\bigr)\Bigr)
+ \frac{\lambda}{2}\|\alpha\|^2 ,
\)
where
\(
\mathds{k}_{t-1}(z)
=
\bigl[\mathds{k}(z,(x_j,x'_j))\bigr]_{j=1}^{t-1}
\)
denotes the vector of dueling-kernel evaluations between the pair $z$ and the past observed pairs. This yields the explicit expression
\[
\forall x \in \cX,\qquad
\hat{f}_t(x)
=
\hat{\theta}_t^\top \Psi(x)
=
\bigl\langle \alpha_t, k_{t-1}(x)-k'_{t-1}(x) \bigr\rangle ,
\]
where
\(
k_{t-1}(x)=\bigl[k(x,x_j)\bigr]_{j=1}^{t-1},
\;
k'_{t-1}(x)=\bigl[k(x,x'_j)\bigr]_{j=1}^{t-1}.
\) As in the linear case, the confidence width over the pair space admits the kernelized form
\begin{align*}
\omega_t(x,x')^2
&\propto
\beta_t
\bigl\|\Psi(x)-\Psi(x')\bigr\|_{H_t^{-1}}^2 \\
&\propto
\frac{\beta_t}{\lambda}
\Bigl(
\mathds{k}(z,z)
-
\mathds{k}_{t-1}(z)^\top
W_{t-1}^{1/2}
\bigl(
W_{t-1}^{1/2}\mathds{K}_{t-1}W_{t-1}^{1/2}
+\lambda I
\bigr)^{-1}
W_{t-1}^{1/2}
\mathds{k}_{t-1}(z)
\Bigr),
\end{align*}
where $z=(x,x')$,
\(
H_t
=
\lambda I_{\cH_{\mathds{k}}}
+
\sum_{s=1}^{t-1}
\dot{\sigma}\bigl(\hat{f}_t(x_s)-\hat{f}_t(x'_s)\bigr)
\bigl(\Psi(x_s)-\Psi(x'_s)\bigr)
\bigl(\Psi(x_s)-\Psi(x'_s)\bigr)^\top
\), $W= \operatorname{diag}(\dot{\sigma}(\Delta_{\hat{f}_t}(z_1)),\ldots,\dot{\sigma}(\Delta_{\hat{f}_t}(z_{t-1}))) $ and
\(
\mathds{K}_{t-1}
=
\bigl[\mathds{k}\bigl((x_i,x_i'),(x_j,x_j')\bigr)\bigr]_{i,j=1}^{t-1}
\)
is the dueling-kernel Gram matrix over the observed pairs.

\propsigmaeluderkern*
\begin{proof}[\cref{prp:propsigmaeluderker}]
Let $\bigl((x_1,x'_1),\dots,(x_n,x'_n)\bigr) \in \mathcal{X}^{2n}$ and $(f_1,\dots,f_n) \subseteq \cF_k$ be the longest $\varepsilon$-independent sequence, i.e $\delud\!\left(\Delta_{\cF_k},\, \Delta_{f^\star},\, \varepsilon\right)=n$. For brevity, we denote $z_i = \Psi(x_i) - \Psi(x'_i)$ and consider $(\theta_1,\dots,\theta_n) \subseteq \ell_2(\lN)$ such that $f_i = \theta_i^\top \Psi$ for all $i \in [n]$. Then there exists $\varepsilon' > \varepsilon$ such that for all $t \leq n$,
\[
\sum_{i=1}^{t-1} \bar{\varphi}(\sigma(z_i^\top\theta^\star),\sigma(z_i^\top\theta_{t})) \leq \varepsilon',
\qquad
\dot{\sigma}\bigl(z_t^\top \theta^\star\bigr) \bigl(z_t^\top (\theta^\star - \theta_t)\bigr)^2 \geq \varepsilon'.
\]

A second-order Taylor expansion yields, for all $i<t \le n$,
\begin{align*}
\bar{\varphi}(\sigma(z_i^\top\theta^\star),\sigma(z_i^\top\theta_{t}))
 = \bar{\varphi}(\sigma(z_i^\top\theta^\star),\sigma(z_i^\top\theta^\star))
 &+ \left[ \sigma(z_i^\top \theta^\star) - \sigma(z_i^\top \theta^\star)\right]z_i^\top (\theta_t-\theta^\star) \\
  &+ \left[z_i^\top (\theta_t - \theta^\star) \right]^2 \underbrace{\int_{0}^1(1-\nu) \dot{\sigma}(z_i^\top \theta^\star + \nu z_i^\top (\theta_t-\theta^\star))d\nu}_{(A)}.
\end{align*}

Using Lemma~8 of~\cite{pmlr-v130-abeille21a}., we obtain
\[
 (A) \geq \frac{\dot{\sigma}(z_i^\top \theta^\star)}{2 + \abs{z_i^T(\theta_t -\theta^\star)}} \ge \frac{\dot{\sigma}(z_i^\top \theta^\star)}{2(1+S)}.
\]
Hence, for all $t \le n$,
\begin{align}
 \sum_{i=1}^{t-1} \dot{\sigma}(z_i^\top \theta^\star) (z_i^\top (\theta^\star-\theta_t))^2 \le 2(1+S)\varepsilon' \label{eq:eq1_exker},
\end{align}
while
\begin{align}
 \dot{\sigma}(z_t^\top \theta^\star) (z_t^\top (\theta^\star-\theta_t))^2 \ge \varepsilon'. \label{eq:eq2_exker}
\end{align}

Define
\[
 \bar{H}_t(\theta^\star) = \sum_{i=1}^{t-1} \dot{\sigma}(z_i^\top \theta^\star) z_i z_i^\top,
 \qquad H_t(\theta^\star) = \bar{H}_t(\theta^\star) +  \lambda I,
\]
with $\lambda = 2\varepsilon(1+S)/(2S)^2$. From equation~\eqref{eq:eq1_exker}:
\[
 \|\theta^\star-\theta_t\|_{H_t(\theta^\star)}^2 \leq \|\theta^\star-\theta_t\|_{\bar{H}_t(\theta^\star)}^2 + (2S)^2 \lambda
 \leq 2(1+S)\varepsilon' + (2S)^2 \lambda
 \leq 2 (1+S)(\varepsilon'+ \varepsilon)
\]

From equation~\eqref{eq:eq2_exker}:
\begin{align*}
 \varepsilon' \leq \dot{\sigma}(z_t^\top \theta_t) \norm{z_t}^2_{H_t(\theta^\star)^{-1}} \norm{\theta_t-\theta^\star}^2_{H_t(\theta^\star)} &\leq 2 (1+S)(\varepsilon'+ \varepsilon) \dot{\sigma}(z_t^\top \theta^\star) \norm{z_t}^2_{H_t(\theta^\star)^{-1}}  \\
 \implies \dot{\sigma}(z_t^\top \theta^\star) \norm{z_t}^2_{H_t(\theta^\star)^{-1}}  &\geq \frac{\varepsilon'}{2 (1+S)(\varepsilon'+ \varepsilon)}\geq  \frac{1}{4 (1+S)}.
\end{align*}
Using that $ \operatorname{det}\Bigl(\lambda^{-1}H_t(\theta^\star)\Bigr)= \operatorname{det}\Bigl(\lambda^{-1}H_{t-1}(\theta^\star)\Bigr)\Big( 1 + \dot{\sigma}(z_t^\top \theta^\star)\norm{z_t}_{H_t(\theta^\star)}^{-1}\Big) $ we obtain:
\[ \operatorname{det}\Bigl(\lambda^{-1}H_n(\theta^\star)\Bigr)\geq \operatorname{det}\Bigl(I\Bigr) \Big( 1 + \frac{1}{4(1+S)}\Big)^{n} \geq \Big( 1 + \frac{1}{4(1+S)}\Big)^{n} \]

Given that $\dot{\sigma}(\cdot) \leq 1$ and the definition of the maximum information gain, we obtain:
\begin{align*}
\gamma_T(\lambda;\, \cX \times \cX)
&\geq \Log{\operatorname{det}\Bigl(\lambda^{-1} H_n(\theta^\star)\Bigr)}, \\
\implies n &\leq \left(\Log{1 + \frac{1}{4(1+S)}}\right)^{-1} \gamma_T(\lambda;\, \cX \times \cX).
\end{align*}
\end{proof}

\subsubsection{Specialization to Commonly Used Kernel Functions}\label{sec:specified_ker}
Combining~\cref{thm:Alg_perf} and~\cref{prp:propsigmaeluderker} with the fact that, for bounded RKHS, $\log \cN_T \asymp \gamma_T$, the regret upper bound of \textbf{GINOP} in the kernelized setting becomes:
\begin{align*}
    \cR_T
&=
\cO\!\left( \gamma_T
\sqrt{\frac{T}{\dot{\sigma}^\star}}
+ \kappa\, \gamma_T^3 \right).
\end{align*}
In what follows, we specialize this result to several commonly used kernel functions.

\paragraph{Linear kernel.} The linear kernel is defined as
\(
k(x, x') = x^\top x'
\) where $x,x'\in \lR^d$
and, according to~\cite{pmlr-v130-vakili21a}, $\gamma_T = \cO(d \log T)$. Hence, $\gamma_T^3 = o(\sqrt{T})$, and consequently
\begin{align}
   \cR_T = \widetilde{\cO}\!\left( \gamma_T \sqrt{\frac{T}{\dot{\sigma}^\star}} \right)= \widetilde{\cO}\!\left( 2d \Log{T} \sqrt{{T}} \right). \label{eq:lin_ker}
\end{align}

\paragraph{Squared exponential (SE) kernel.} The SE kernel is defined as
\[
k(x, x') = \sigma_{\mathrm{SE}}^2 \exp\!\left( - \frac{\|x - x'\|^2}{\ell^2} \right), \qquad x, x' \in \lR^d,
\]
where $\sigma_{\mathrm{SE}}^2$ denotes the variance parameter and $\ell$ the length-scale parameter. According to~\cite{pmlr-v130-vakili21a}, $\gamma_T = \cO\!\left( \log^{d+1}(T) \right)$. Hence, $\gamma_T^3 = o(\sqrt{T})$, and consequently
\begin{align}
   \cR_T = \widetilde{\cO}\!\left( \gamma_T \sqrt{\frac{T}{\dot{\sigma}^\star}} \right) = \widetilde{\cO}\!\left( 2\log^{d+1}(T) \sqrt{{T}} \right). \label{eq:se_ker}
\end{align}

\paragraph{Matérn kernel.} The Matérn kernel is defined as
\[
k(x, x') = \frac{2^{1-\nu}}{\Gamma(\nu)} 
\left( \frac{\sqrt{2\nu}\,\|x - x'\|}{\rho} \right)^\nu 
K_\nu\!\left( \frac{\sqrt{2\nu}\,\|x - x'\|}{\rho} \right), \qquad x, x' \in \lR^d,
\]
where $\rho, \nu > 0$ are kernel parameters, $\Gamma(\cdot)$ denotes the gamma function, and $K_\nu(\cdot)$ is the modified Bessel function of the second kind. The parameter $\nu$ controls the smoothness of the kernel and, following~\cite{Dueling-xu2024-kernel}, is assumed to be large enough such that $\nu > 3d/2$. According to~\cite{pmlr-v130-vakili21a}, $\gamma_T = \cO\!\left( T^{\frac{d}{2\nu + d}} \Log{T} \right)$. Hence, $\gamma_T^3 = o(\gamma_T \sqrt{T})$, and consequently
\begin{align}
   \cR_T = \widetilde{\cO}\!\left( \gamma_T \sqrt{\frac{T}{\dot{\sigma}^\star}} \right) = \widetilde{\cO}\!\left( 2 \Log{T} T^{\frac{1}{2} + \frac{d}{2\nu + d}} \right). \label{eq:matern_ker}
\end{align}

As a result, by~\cref{eq:lin_ker,eq:se_ker,eq:matern_ker} and in line with~\cite{kayal2025bayesian}, \textbf{GINOP} exhibits the correct dependence on $\kappa$ and improves upon the bounds of~\cite{Dueling-xu2024-kernel} and~\cite{Dueling_Stackelburg_Krause}, which are of order $\widetilde{\cO}\bigl((\gamma_T T)^{3/4}\bigr)$ and $\widetilde{\cO}\bigl(\gamma_T \kappa^2 \sqrt{T}\bigr)$, respectively.

\section{Details of Experiments}\phantomsection\label{app:experiments}

To corroborate our theoretical results, we conduct numerical experiments evaluating the performance of \textbf{GINOP} against several existing baselines across a range of reward function classes. For all experiments, we report the cumulative regret over a horizon of $T = 2000$ rounds, averaged over $20$ independent trials. All experiments were conducted on a machine featuring an Apple M1 chip (8 cores) and 16 GB of RAM.

\paragraph{Linear Setting.} We consider the linear utility setting in which $f^\star(x) = \langle x, \theta^\star \rangle$ for all $x \in \cX$, with $\theta^\star \in \mathbb{B}_2^d(S)$, $d \in \{10,15,20\}$, $S = 2$, and $\abs{\cX} = 30$, with the arms drawn uniformly at random from $\{\frac{-1}{\sqrt{d}}, \frac{1}{\sqrt{d}}\}^d$. We benchmark against \textbf{MaxInP}~\citep{Dueling-saha-linear}, \textbf{CoLSTIM}~\citep{pmlr-v162-bengs22a}, and \textbf{FGTS.CDB}~\citep{li2024feelgoodthompsonsamplingcontextual}, which, despite their differing decision rules, all attain a regret bound of order $\kappa d \sqrt{T}$. While \textbf{MaxInP} requires no specific hyperparameter tuning, for \textbf{CoLSTIM} we follow the recommendations of the original paper and set $c = C_{\mathrm{thresh}} = \eta = \sqrt{d \log T}$, $\tau = t_0 = d \abs{\cX}$, and 
\(
p_t = \min\!\left(1,\, \sqrt{\frac{d}{t - \tau} \log(dT)}\right);
\)
for \textbf{FGTS.CDB}, in accordance with their Theorem~5.2, we set $\eta = 0.25$ and $\mu = 1 / (10 e^S \sqrt{T})$. The results, reported in~\cref{fig:LinUT}, indicate that \textbf{GINOP} consistently outperforms the competing baselines: due to their adverse dependence on $\kappa$, the baselines' regret is inflated, exhibiting a $\sqrt{T}$-like shape that becomes apparent only at considerably larger horizons.


\begin{figure}[h]
    \centering
    \begin{subfigure}[b]{0.32\linewidth}
        \centering
        \includegraphics[width=\linewidth]{Figure/cumulative_regret_K30_d10_T2000_epochs20_S2.png}
        \caption{$d=10$.}
        \label{fig:Linfirst}
    \end{subfigure}
    \hfill
    \begin{subfigure}[b]{0.32\linewidth}
        \centering
        \includegraphics[width=\linewidth]{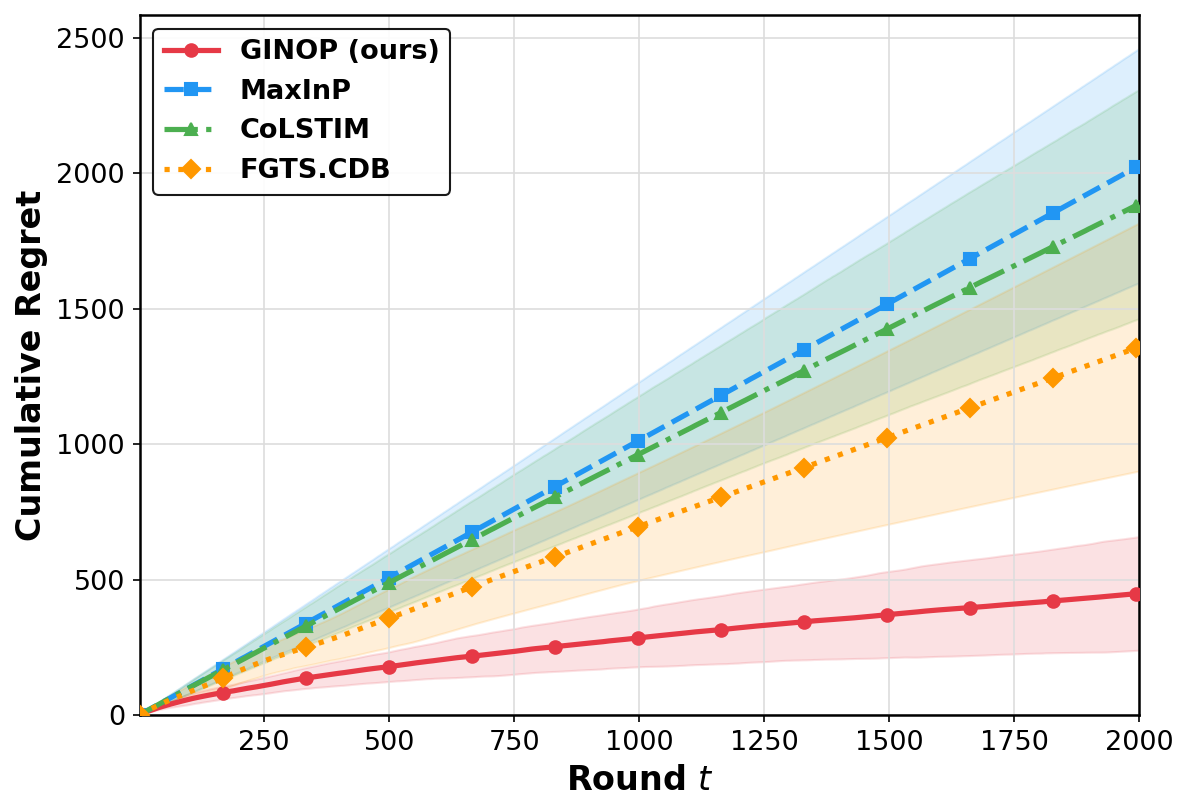}
        \caption{$d=15$.}
        \label{fig:Linsecond}
    \end{subfigure}
    \hfill
    \begin{subfigure}[b]{0.32\linewidth}
        \centering
        \includegraphics[width=\linewidth]{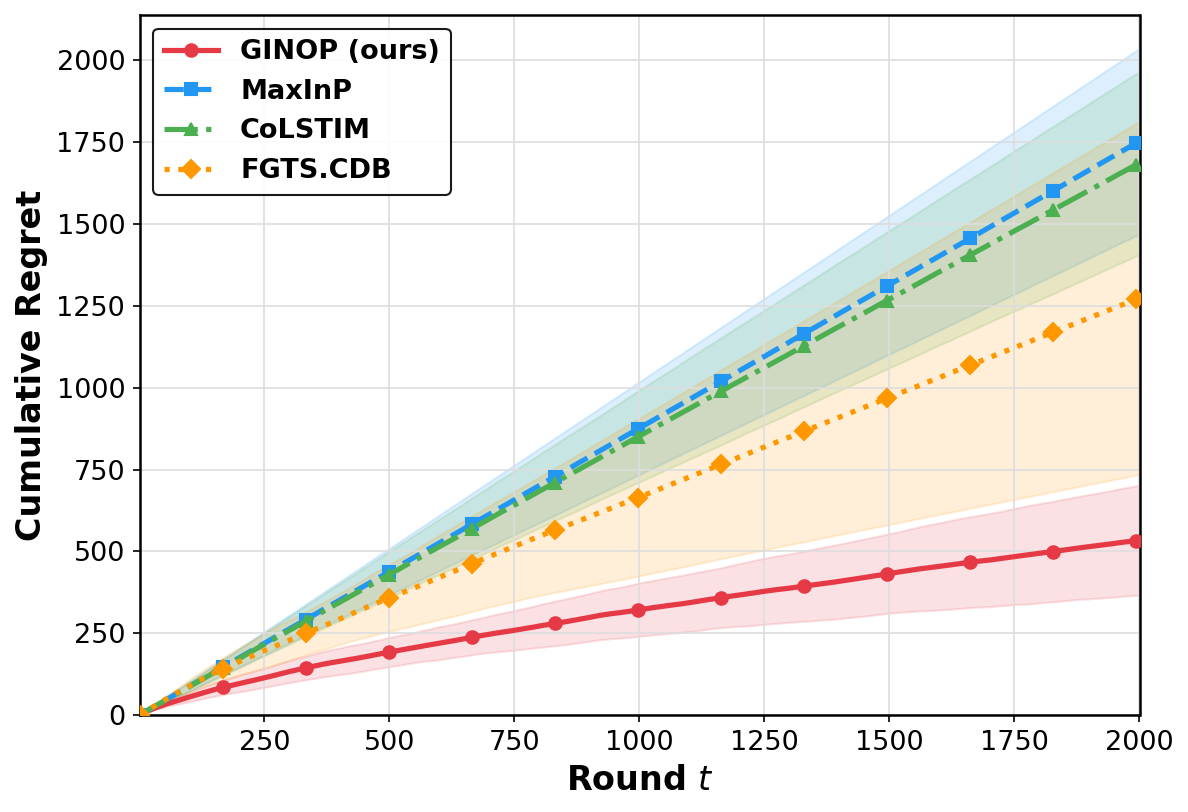}
        \caption{$d=20$.}
        \label{fig:Linthird}
    \end{subfigure}
\caption{Linear Utility.}
    \label{fig:LinUT}
\end{figure}

\paragraph{Kernelized Setting.} We adopt the Ackley function as the reward on the interval $[-5, 5]$. The Ackley function has a diverse optimization landscape, featuring multiple local minima, flat plateaus, and valleys, which makes it a popular benchmark in the non-convex optimization literature. It is defined as (here we consider $d=1$):
\[
f(x) = -20 \exp\!\left(-0.2 \sqrt{\frac{1}{d} \sum_{i=1}^d x_i^2}\right) - \exp\!\left(\frac{1}{d} \sum_{i=1}^d \cos(2\pi x_i)\right) + 20 + \exp(1).
\]
We consider $30$ arms forming a uniform mesh over the input domain, and employ the Matérn kernel with smoothness parameter $\nu = 2.5$ and lengthscale $\rho = 0.1$. We benchmark against \textbf{POP-BO}~\citep{Dueling-xu2024-kernel}, \textbf{MaxMinLCB}~\citep{Dueling_Stackelburg_Krause}, and \textbf{MR-LPF}~\citep{kayal2025bayesian}. The results are reported in~\cref{fig:KerUT}. \textbf{GINOP} clearly outperforms \textbf{POP-BO} and \textbf{MaxMinLCB}: the former scales as $(\gamma_TT)^{3/4}$, while the latter incurs an additional adverse dependence on $\kappa$ in the leading term. \textbf{MR-LPF}, in contrast, exhibits a regret profile comparable to that of \textbf{GINOP}, as its multi-phase approach allows it to eliminate the $\kappa$-dependence.

\begin{figure}[ht]
    \centering
    \includegraphics[width=0.5\linewidth]{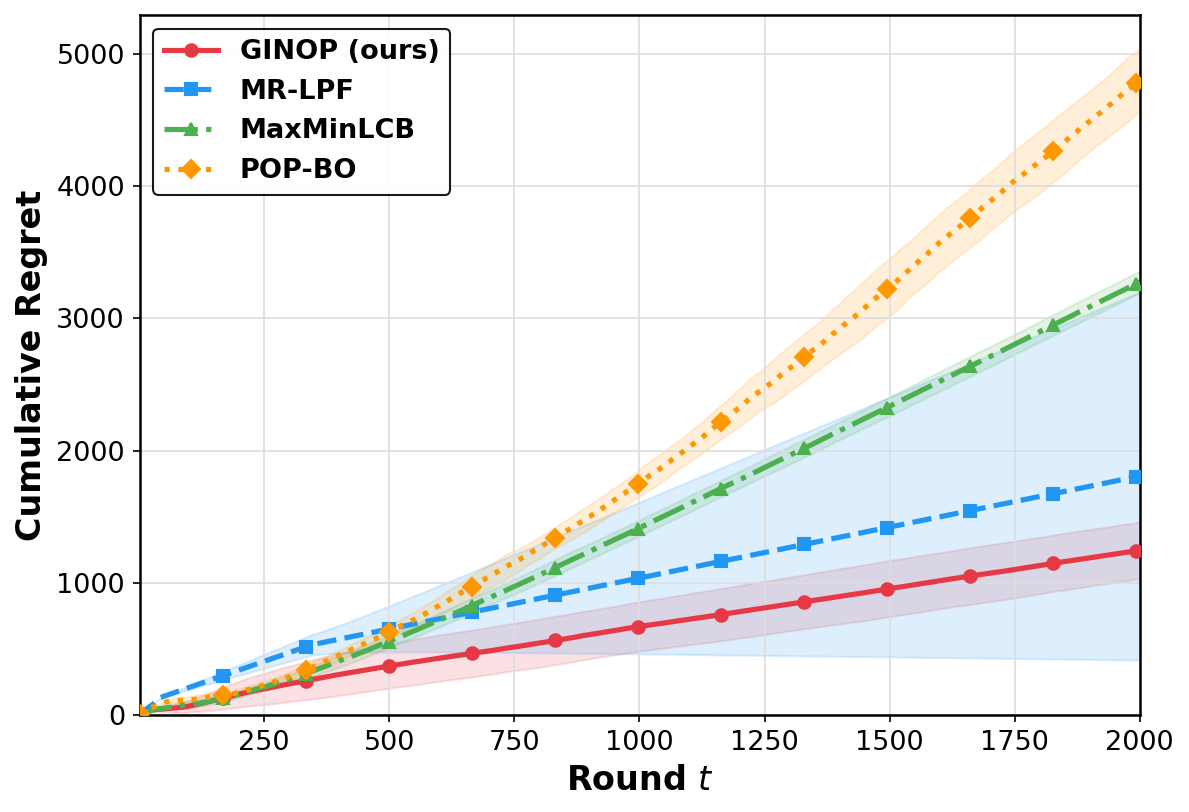}
    \caption{Kernelized Utility.}
    \label{fig:KerUT}
\end{figure}

\paragraph{Against a Neural-Network Approach.} In this setting, the reward is generated by $f^\star(x) = \cos(3 \langle x, \theta^\star \rangle)$, with $\theta^\star \in \mathbb{R}^d$, where we take $d = 10$ and $\abs{\cX} = 30$, with the arms drawn uniformly at random from $\{-1, 1\}^d$. We benchmark against \textbf{NDB-UCB} and \textbf{NDB-TS}~\citep{verma2025neural}, which employ a neural network in the NTK regime to estimate the reward function and select arms accordingly. Following the original paper, we use a neural network with $2$ hidden layers of width $50$ and ReLU activation functions, with hyperparameters $\lambda = 1.0$, $\delta = 0.05$, and a fixed $\nu_T = \nu = 1.0$. The results, reported in~\cref{fig:NnUT}, show that \textbf{GINOP} outperforms both \textbf{NDB-UCB} and \textbf{NDB-TS}, in line with the theory: the regret bounds of \textbf{NDB-UCB} and \textbf{NDB-TS} exhibit an adverse dependence on $\kappa$.
\begin{figure}[ht]
    \centering
    \includegraphics[width=0.5\linewidth]{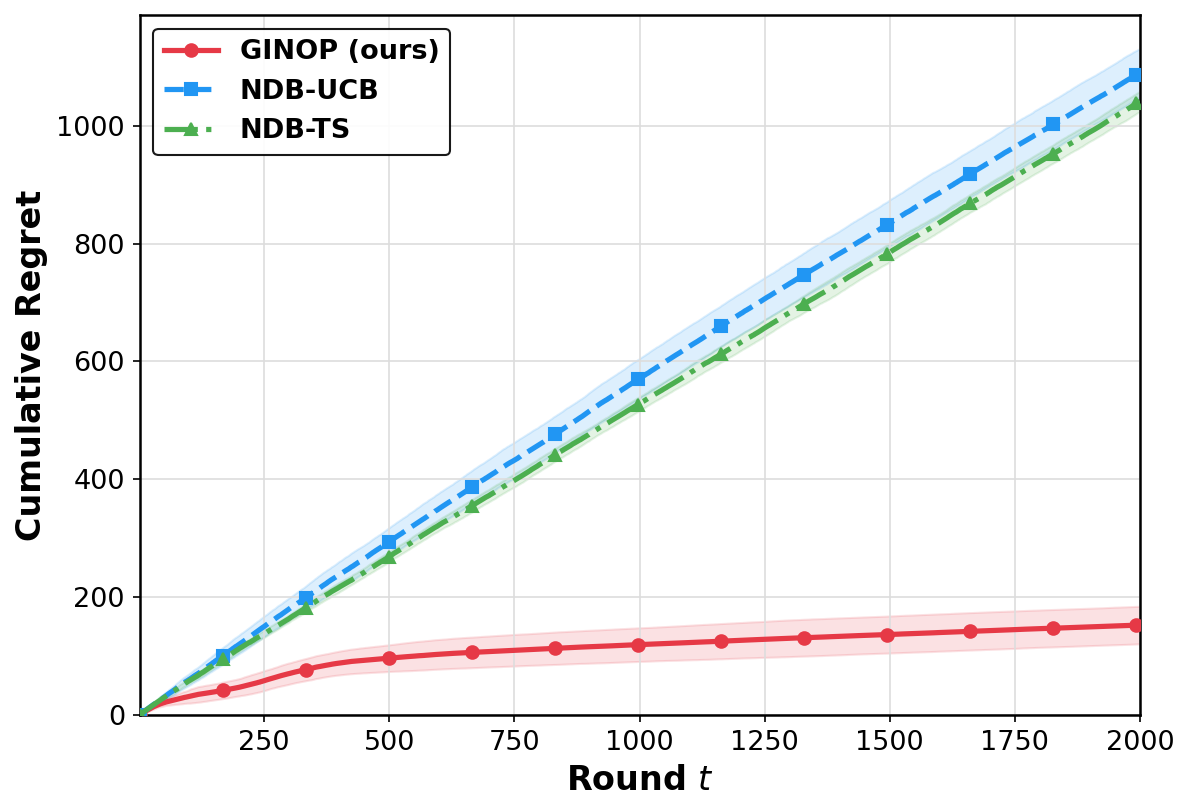}
    \caption{Against a NN.}
    \label{fig:NnUT}
\end{figure}

\newpage
\clearpage
\section*{NeurIPS Paper Checklist}


\begin{enumerate}

\item {\bf Claims}
    \item[] Question: Do the main claims made in the abstract and introduction accurately reflect the paper's contributions and scope?
    \item[] Answer: \answerYes{} 
    \item[] Justification: 
    \item[] Guidelines:
    \begin{itemize}
        \item The answer \answerNA{} means that the abstract and introduction do not include the claims made in the paper.
        \item The abstract and/or introduction should clearly state the claims made, including the contributions made in the paper and important assumptions and limitations. A \answerNo{} or \answerNA{} answer to this question will not be perceived well by the reviewers. 
        \item The claims made should match theoretical and experimental results, and reflect how much the results can be expected to generalize to other settings. 
        \item It is fine to include aspirational goals as motivation as long as it is clear that these goals are not attained by the paper. 
    \end{itemize}

\item {\bf Limitations}
    \item[] Question: Does the paper discuss the limitations of the work performed by the authors?
    \item[] Answer: \answerYes{} 
    \item[] Justification:
    \item[] Guidelines:
    \begin{itemize}
        \item The answer \answerNA{} means that the paper has no limitation while the answer \answerNo{} means that the paper has limitations, but those are not discussed in the paper. 
        \item The authors are encouraged to create a separate ``Limitations'' section in their paper.
        \item The paper should point out any strong assumptions and how robust the results are to violations of these assumptions (e.g., independence assumptions, noiseless settings, model well-specification, asymptotic approximations only holding locally). The authors should reflect on how these assumptions might be violated in practice and what the implications would be.
        \item The authors should reflect on the scope of the claims made, e.g., if the approach was only tested on a few datasets or with a few runs. In general, empirical results often depend on implicit assumptions, which should be articulated.
        \item The authors should reflect on the factors that influence the performance of the approach. For example, a facial recognition algorithm may perform poorly when image resolution is low or images are taken in low lighting. Or a speech-to-text system might not be used reliably to provide closed captions for online lectures because it fails to handle technical jargon.
        \item The authors should discuss the computational efficiency of the proposed algorithms and how they scale with dataset size.
        \item If applicable, the authors should discuss possible limitations of their approach to address problems of privacy and fairness.
        \item While the authors might fear that complete honesty about limitations might be used by reviewers as grounds for rejection, a worse outcome might be that reviewers discover limitations that aren't acknowledged in the paper. The authors should use their best judgment and recognize that individual actions in favor of transparency play an important role in developing norms that preserve the integrity of the community. Reviewers will be specifically instructed to not penalize honesty concerning limitations.
    \end{itemize}

\item {\bf Theory assumptions and proofs}
    \item[] Question: For each theoretical result, does the paper provide the full set of assumptions and a complete (and correct) proof?
    \item[] Answer: \answerYes{} 
    \item[] Justification: 
    \item[] Guidelines:
    \begin{itemize}
        \item The answer \answerNA{} means that the paper does not include theoretical results. 
        \item All the theorems, formulas, and proofs in the paper should be numbered and cross-referenced.
        \item All assumptions should be clearly stated or referenced in the statement of any theorems.
        \item The proofs can either appear in the main paper or the supplemental material, but if they appear in the supplemental material, the authors are encouraged to provide a short proof sketch to provide intuition. 
        \item Inversely, any informal proof provided in the core of the paper should be complemented by formal proofs provided in appendix or supplemental material.
        \item Theorems and Lemmas that the proof relies upon should be properly referenced. 
    \end{itemize}

    \item {\bf Experimental result reproducibility}
    \item[] Question: Does the paper fully disclose all the information needed to reproduce the main experimental results of the paper to the extent that it affects the main claims and/or conclusions of the paper (regardless of whether the code and data are provided or not)?
    \item[] Answer: \answerYes{} 
    \item[] Justification: 
    \item[] Guidelines:
    \begin{itemize}
        \item The answer \answerNA{} means that the paper does not include experiments.
        \item If the paper includes experiments, a \answerNo{} answer to this question will not be perceived well by the reviewers: Making the paper reproducible is important, regardless of whether the code and data are provided or not.
        \item If the contribution is a dataset and\slash or model, the authors should describe the steps taken to make their results reproducible or verifiable. 
        \item Depending on the contribution, reproducibility can be accomplished in various ways. For example, if the contribution is a novel architecture, describing the architecture fully might suffice, or if the contribution is a specific model and empirical evaluation, it may be necessary to either make it possible for others to replicate the model with the same dataset, or provide access to the model. In general. releasing code and data is often one good way to accomplish this, but reproducibility can also be provided via detailed instructions for how to replicate the results, access to a hosted model (e.g., in the case of a large language model), releasing of a model checkpoint, or other means that are appropriate to the research performed.
        \item While NeurIPS does not require releasing code, the conference does require all submissions to provide some reasonable avenue for reproducibility, which may depend on the nature of the contribution. For example
        \begin{enumerate}
            \item If the contribution is primarily a new algorithm, the paper should make it clear how to reproduce that algorithm.
            \item If the contribution is primarily a new model architecture, the paper should describe the architecture clearly and fully.
            \item If the contribution is a new model (e.g., a large language model), then there should either be a way to access this model for reproducing the results or a way to reproduce the model (e.g., with an open-source dataset or instructions for how to construct the dataset).
            \item We recognize that reproducibility may be tricky in some cases, in which case authors are welcome to describe the particular way they provide for reproducibility. In the case of closed-source models, it may be that access to the model is limited in some way (e.g., to registered users), but it should be possible for other researchers to have some path to reproducing or verifying the results.
        \end{enumerate}
    \end{itemize}

\item {\bf Open access to data and code}
    \item[] Question: Does the paper provide open access to the data and code, with sufficient instructions to faithfully reproduce the main experimental results, as described in supplemental material?
    \item[] Answer: \answerNA{} 
    \item[] Justification: 
    \item[] Guidelines:
    \begin{itemize}
        \item The answer \answerNA{} means that paper does not include experiments requiring code.
        \item Please see the NeurIPS code and data submission guidelines (\url{https://neurips.cc/public/guides/CodeSubmissionPolicy}) for more details.
        \item While we encourage the release of code and data, we understand that this might not be possible, so \answerNo{} is an acceptable answer. Papers cannot be rejected simply for not including code, unless this is central to the contribution (e.g., for a new open-source benchmark).
        \item The instructions should contain the exact command and environment needed to run to reproduce the results. See the NeurIPS code and data submission guidelines (\url{https://neurips.cc/public/guides/CodeSubmissionPolicy}) for more details.
        \item The authors should provide instructions on data access and preparation, including how to access the raw data, preprocessed data, intermediate data, and generated data, etc.
        \item The authors should provide scripts to reproduce all experimental results for the new proposed method and baselines. If only a subset of experiments are reproducible, they should state which ones are omitted from the script and why.
        \item At submission time, to preserve anonymity, the authors should release anonymized versions (if applicable).
        \item Providing as much information as possible in supplemental material (appended to the paper) is recommended, but including URLs to data and code is permitted.
    \end{itemize}

\item {\bf Experimental setting/details}
    \item[] Question: Does the paper specify all the training and test details (e.g., data splits, hyperparameters, how they were chosen, type of optimizer) necessary to understand the results?
    \item[] Answer: \answerYes{} 
    \item[] Justification: 
    \item[] Guidelines:
    \begin{itemize}
        \item The answer \answerNA{} means that the paper does not include experiments.
        \item The experimental setting should be presented in the core of the paper to a level of detail that is necessary to appreciate the results and make sense of them.
        \item The full details can be provided either with the code, in appendix, or as supplemental material.
    \end{itemize}

\item {\bf Experiment statistical significance}
    \item[] Question: Does the paper report error bars suitably and correctly defined or other appropriate information about the statistical significance of the experiments?
    \item[] Answer: \answerYes{} 
    \item[] Justification: 
    \item[] Guidelines:
    \begin{itemize}
        \item The answer \answerNA{} means that the paper does not include experiments.
        \item The authors should answer \answerYes{} if the results are accompanied by error bars, confidence intervals, or statistical significance tests, at least for the experiments that support the main claims of the paper.
        \item The factors of variability that the error bars are capturing should be clearly stated (for example, train/test split, initialization, random drawing of some parameter, or overall run with given experimental conditions).
        \item The method for calculating the error bars should be explained (closed form formula, call to a library function, bootstrap, etc.)
        \item The assumptions made should be given (e.g., Normally distributed errors).
        \item It should be clear whether the error bar is the standard deviation or the standard error of the mean.
        \item It is OK to report 1-sigma error bars, but one should state it. The authors should preferably report a 2-sigma error bar than state that they have a 96\% CI, if the hypothesis of Normality of errors is not verified.
        \item For asymmetric distributions, the authors should be careful not to show in tables or figures symmetric error bars that would yield results that are out of range (e.g., negative error rates).
        \item If error bars are reported in tables or plots, the authors should explain in the text how they were calculated and reference the corresponding figures or tables in the text.
    \end{itemize}

\item {\bf Experiments compute resources}
    \item[] Question: For each experiment, does the paper provide sufficient information on the computer resources (type of compute workers, memory, time of execution) needed to reproduce the experiments?
    \item[] Answer: \answerYes{} 
    \item[] Justification: 
    \item[] Guidelines:
    \begin{itemize}
        \item The answer \answerNA{} means that the paper does not include experiments.
        \item The paper should indicate the type of compute workers CPU or GPU, internal cluster, or cloud provider, including relevant memory and storage.
        \item The paper should provide the amount of compute required for each of the individual experimental runs as well as estimate the total compute. 
        \item The paper should disclose whether the full research project required more compute than the experiments reported in the paper (e.g., preliminary or failed experiments that didn't make it into the paper). 
    \end{itemize}
    
\item {\bf Code of ethics}
    \item[] Question: Does the research conducted in the paper conform, in every respect, with the NeurIPS Code of Ethics \url{https://neurips.cc/public/EthicsGuidelines}?
    \item[] Answer: \answerYes{} 
    \item[] Justification: 
    \item[] Guidelines:
    \begin{itemize}
        \item The answer \answerNA{} means that the authors have not reviewed the NeurIPS Code of Ethics.
        \item If the authors answer \answerNo, they should explain the special circumstances that require a deviation from the Code of Ethics.
        \item The authors should make sure to preserve anonymity (e.g., if there is a special consideration due to laws or regulations in their jurisdiction).
    \end{itemize}

\item {\bf Broader impacts}
    \item[] Question: Does the paper discuss both potential positive societal impacts and negative societal impacts of the work performed?
    \item[] Answer: \answerNA{} 
    \item[] Justification: 
    \item[] Guidelines:
    \begin{itemize}
        \item The answer \answerNA{} means that there is no societal impact of the work performed.
        \item If the authors answer \answerNA{} or \answerNo, they should explain why their work has no societal impact or why the paper does not address societal impact.
        \item Examples of negative societal impacts include potential malicious or unintended uses (e.g., disinformation, generating fake profiles, surveillance), fairness considerations (e.g., deployment of technologies that could make decisions that unfairly impact specific groups), privacy considerations, and security considerations.
        \item The conference expects that many papers will be foundational research and not tied to particular applications, let alone deployments. However, if there is a direct path to any negative applications, the authors should point it out. For example, it is legitimate to point out that an improvement in the quality of generative models could be used to generate Deepfakes for disinformation. On the other hand, it is not needed to point out that a generic algorithm for optimizing neural networks could enable people to train models that generate Deepfakes faster.
        \item The authors should consider possible harms that could arise when the technology is being used as intended and functioning correctly, harms that could arise when the technology is being used as intended but gives incorrect results, and harms following from (intentional or unintentional) misuse of the technology.
        \item If there are negative societal impacts, the authors could also discuss possible mitigation strategies (e.g., gated release of models, providing defenses in addition to attacks, mechanisms for monitoring misuse, mechanisms to monitor how a system learns from feedback over time, improving the efficiency and accessibility of ML).
    \end{itemize}
    
\item {\bf Safeguards}
    \item[] Question: Does the paper describe safeguards that have been put in place for responsible release of data or models that have a high risk for misuse (e.g., pre-trained language models, image generators, or scraped datasets)?
    \item[] Answer: \answerNA{} 
    \item[] Justification: 
    \item[] Guidelines:
    \begin{itemize}
        \item The answer \answerNA{} means that the paper poses no such risks.
        \item Released models that have a high risk for misuse or dual-use should be released with necessary safeguards to allow for controlled use of the model, for example by requiring that users adhere to usage guidelines or restrictions to access the model or implementing safety filters. 
        \item Datasets that have been scraped from the Internet could pose safety risks. The authors should describe how they avoided releasing unsafe images.
        \item We recognize that providing effective safeguards is challenging, and many papers do not require this, but we encourage authors to take this into account and make a best faith effort.
    \end{itemize}

\item {\bf Licenses for existing assets}
    \item[] Question: Are the creators or original owners of assets (e.g., code, data, models), used in the paper, properly credited and are the license and terms of use explicitly mentioned and properly respected?
    \item[] Answer: \answerNA{} 
    \item[] Justification: 
    \item[] Guidelines:
    \begin{itemize}
        \item The answer \answerNA{} means that the paper does not use existing assets.
        \item The authors should cite the original paper that produced the code package or dataset.
        \item The authors should state which version of the asset is used and, if possible, include a URL.
        \item The name of the license (e.g., CC-BY 4.0) should be included for each asset.
        \item For scraped data from a particular source (e.g., website), the copyright and terms of service of that source should be provided.
        \item If assets are released, the license, copyright information, and terms of use in the package should be provided. For popular datasets, \url{paperswithcode.com/datasets} has curated licenses for some datasets. Their licensing guide can help determine the license of a dataset.
        \item For existing datasets that are re-packaged, both the original license and the license of the derived asset (if it has changed) should be provided.
        \item If this information is not available online, the authors are encouraged to reach out to the asset's creators.
    \end{itemize}

\item {\bf New assets}
    \item[] Question: Are new assets introduced in the paper well documented and is the documentation provided alongside the assets?
    \item[] Answer: \answerNA{} 
    \item[] Justification: 
    \item[] Guidelines:
    \begin{itemize}
        \item The answer \answerNA{} means that the paper does not release new assets.
        \item Researchers should communicate the details of the dataset\slash code\slash model as part of their submissions via structured templates. This includes details about training, license, limitations, etc. 
        \item The paper should discuss whether and how consent was obtained from people whose asset is used.
        \item At submission time, remember to anonymize your assets (if applicable). You can either create an anonymized URL or include an anonymized zip file.
    \end{itemize}

\item {\bf Crowdsourcing and research with human subjects}
    \item[] Question: For crowdsourcing experiments and research with human subjects, does the paper include the full text of instructions given to participants and screenshots, if applicable, as well as details about compensation (if any)? 
    \item[] Answer: \answerNA{} 
    \item[] Justification: 
    \item[] Guidelines:
    \begin{itemize}
        \item The answer \answerNA{} means that the paper does not involve crowdsourcing nor research with human subjects.
        \item Including this information in the supplemental material is fine, but if the main contribution of the paper involves human subjects, then as much detail as possible should be included in the main paper. 
        \item According to the NeurIPS Code of Ethics, workers involved in data collection, curation, or other labor should be paid at least the minimum wage in the country of the data collector. 
    \end{itemize}

\item {\bf Institutional review board (IRB) approvals or equivalent for research with human subjects}
    \item[] Question: Does the paper describe potential risks incurred by study participants, whether such risks were disclosed to the subjects, and whether Institutional Review Board (IRB) approvals (or an equivalent approval/review based on the requirements of your country or institution) were obtained?
    \item[] Answer: \answerNA{} 
    \item[] Justification: 
    \item[] Guidelines:
    \begin{itemize}
        \item The answer \answerNA{} means that the paper does not involve crowdsourcing nor research with human subjects.
        \item Depending on the country in which research is conducted, IRB approval (or equivalent) may be required for any human subjects research. If you obtained IRB approval, you should clearly state this in the paper. 
        \item We recognize that the procedures for this may vary significantly between institutions and locations, and we expect authors to adhere to the NeurIPS Code of Ethics and the guidelines for their institution. 
        \item For initial submissions, do not include any information that would break anonymity (if applicable), such as the institution conducting the review.
    \end{itemize}

\item {\bf Declaration of LLM usage}
    \item[] Question: Does the paper describe the usage of LLMs if it is an important, original, or non-standard component of the core methods in this research? Note that if the LLM is used only for writing, editing, or formatting purposes and does \emph{not} impact the core methodology, scientific rigor, or originality of the research, declaration is not required.
    \item[] Answer: \answerNA{} 
    \item[] Justification: 
    \item[] Guidelines:
    \begin{itemize}
        \item The answer \answerNA{} means that the core method development in this research does not involve LLMs as any important, original, or non-standard components.
        \item Please refer to our LLM policy in the NeurIPS handbook for what should or should not be described.
    \end{itemize}

\end{enumerate}
\end{document}